\documentclass{article}
\PassOptionsToPackage{numbers, compress}{natbib}
\usepackage[preprint]{neurips_2026}

\usepackage{wrapfig}
\usepackage{graphicx}
\usepackage[utf8]{inputenc}
\usepackage[T1]{fontenc}
\usepackage{hyperref}
\usepackage{url}
\usepackage{booktabs}
\usepackage{amsfonts}
\usepackage{amsmath}
\usepackage{amssymb}
\usepackage{amsthm}
\usepackage{nicefrac}
\usepackage{microtype}
\usepackage{xcolor}
\usepackage{graphicx}
\usepackage{bm}
\usepackage{enumitem}    % for [(i)]-style enumerate
\usepackage{multirow}

\newtheorem{proposition}{Proposition}
\newtheorem{lemma}{Lemma}

\newtheorem{remark}{Remark}

\newcommand{\R}{\mathbb{R}}

\newcommand{\simplex}{\Delta}                      % use \simplex^{N-1} etc.
\newcommand{\Tstar}{\mathbf{T}^\star}

\newcommand{\Hb}{\mathbf{H}}

\newcommand{\bPhi}{\boldsymbol{\Phi}}

\title{Beyond Similarity: Temporal Operator Attention \\ for Time Series Analysis}

\author{
\makebox[\textwidth][c]{
  Jevon Twitty\textsuperscript{*1},
  Vinh Pham\textsuperscript{*1},
  Nitiwith Rotchanarak\textsuperscript{1},
  Viresh Pati\textsuperscript{1},
  Yubin Kim\textsuperscript{1},
  Shihao Yang \textsuperscript{\ddag1},
  Jiecheng Lu\footnotemark[2]\textsuperscript{\dag1}
} \\
\makebox[\textwidth][c]{
  Georgia Institute of Technology\textsuperscript{1}
} \\
\makebox[\textwidth][c]{
  \texttt{jlu414@gatech.edu\textsuperscript{\dag},
  shihao.yang@isye.gatech.edu\textsuperscript{\ddag}}
}
}

\begin{document}

\maketitle

\begingroup
\renewcommand\thefootnote{*}
\footnotetext{Both authors contributed equally to the paper.}
\renewcommand\thefootnote{\dag\ddag}
\footnotetext{Corresponding Authors}
\endgroup

\begin{abstract}
A persistent paradox in time-series forecasting is that structurally simple MLP and linear models often outperform high-capacity Transformers. We argue that this gap arises from a mismatch in the sequence-modeling primitive: while many time-series dynamics are governed by global temporal operators (e.g., filtering and harmonic structure), standard attention forms each output as a convex combination of inputs. This restricts its ability to represent signed and oscillatory transformations that are fundamental to temporal signal processing. We formalize this limitation as a \emph{simplex-constrained mixing bottleneck} in softmax attention, which becomes especially restrictive for operator-driven time-series tasks. To address this, we propose \textbf{Temporal Operator Attention (TOA)}, a framework that augments attention with explicit, learnable sequence-space operators, enabling direct signed mixing across time while preserving input-dependent adaptivity. To make dense $N \times N$ operators practical, we introduce \emph{Stochastic Operator Regularization}, a high-variance dropout mechanism that stabilizes training and prevents trivial memorization. Across forecasting, anomaly detection, and classification benchmarks, TOA consistently improves performance when integrated into standard backbones such as PatchTST and iTransformer, with particularly strong gains in reconstruction-heavy tasks. These results suggest that explicit operator learning is a key ingredient for effective time-series modeling. The code implementation is available at this \href{https://anonymous.4open.science/r/temporal-operator-attention-37DB/README.md}{link}.
\end{abstract}

\section{Introduction}

Transformers have become a dominant architecture for sequence modeling since \citet{vaswani2017attention}. In time-series forecasting, this paradigm has inspired a large family of specialized architectures~\citep{zhou2021informer, wu2021autoformer, zhou2022fedformerfrequencyenhanceddecomposed, nie2023timeseriesworth64, liu2024itransformer, qiu2025duetdualclusteringenhanced}. Yet despite this progress, a persistent and somewhat paradoxical empirical result remains: simple linear and MLP-based models often match or outperform Transformer-based forecasters on long-horizon benchmarks~\citep{zeng2022transformerseffectivetimeseries, chen2023tsmixerallmlparchitecturetime}.

We argue that this gap arises from a mismatch in the underlying sequence-modeling primitive. Standard self-attention forms each output as a normalized convex combination of input tokens~\citep{vaswani2017attention}. This enables flexible, input-dependent routing, but it does not directly parameterize signed transformations over the sequence axis. In contrast, many time-series dynamics are more naturally expressed as global temporal operators---including filtering, harmonic continuation, residualization, and demixing---which often require signed or zero-sum interactions fundamental to classical signal processing~\citep{10.5555/294797}. This distinction is also consistent with the broader empirical success of direct sequence-mixing architectures in time-series modeling~\citep{zeng2022transformerseffectivetimeseries, chen2023tsmixerallmlparchitecturetime, nie2023timeseriesworth64, gu2022efficientlymodelinglongsequences}.

This structural limitation prevents softmax attention from directly realizing signed, oscillatory, or zero-sum sequence operators at the primitive level. We formalize this limitation as a \emph{simplex-constrained mixing bottleneck}. Our claim is not that deep multi-head Transformers are universally incapable of approximating such behavior, but rather that the softmax attention primitive imposes an unfavorable inductive bias for operator-driven temporal tasks. In this respect, our perspective is complementary to prior analyses showing that pure attention layers exhibit a bias toward token uniformity and rank degeneration in the absence of residual and MLP pathways~\citep{dong2023attentionneedpureattention}.

To address this mismatch, we propose \textbf{Temporal Operator Attention (TOA)}, a modification of the attention mechanism that augments similarity-based routing with explicit, learnable sequence-space operators. Rather than relying solely on nonnegative similarity matching, TOA introduces dense signed transformations over the temporal axis, enabling direct modeling of operator-driven dynamics while preserving input-dependent adaptivity. At a high level, this view is also compatible with operator-theoretic perspectives on representation learning, where complex dynamics are modeled through linear structure in an appropriate latent space~\citep{Lusch_2018}.

A key challenge in learning dense $N \times N$ operators is overfitting. We address this through \textbf{Stochastic Operator Regularization}, a high-variance dropout mechanism applied directly to the learnable sequence offsets while preserving the residual identity path. This stabilizes training and encourages the model to learn robust structural filters rather than brittle instance-specific patterns.

We evaluate TOA across forecasting, anomaly detection, and classification benchmarks, integrating it into standard backbones such as PatchTST~\citep{nie2023timeseriesworth64}, iTransformer~\citep{liu2024itransformer}, and DUET~\citep{qiu2025duetdualclusteringenhanced}. Across these settings, TOA consistently improves performance, with particularly strong gains in reconstruction-heavy tasks. These results support the view that explicit operator learning is a key ingredient for effective time-series modeling.

\textbf{Contributions.}
We (1) identify a simplex-constrained mixing bottleneck in softmax attention for operator-driven time-series modeling, (2) connect this limitation to the absence of explicit sequence-space operator parameterization, (3) propose TOA to bridge similarity-based routing with explicit dense temporal operators, and (4) introduce Stochastic Operator Regularization to enable stable learning of dense operators, yielding consistent gains across multiple time-series tasks.

\begin{figure}[t]
    \centering
    % Adjust width as needed (e.g., 1.0\linewidth)
    \includegraphics[width=0.95\linewidth]{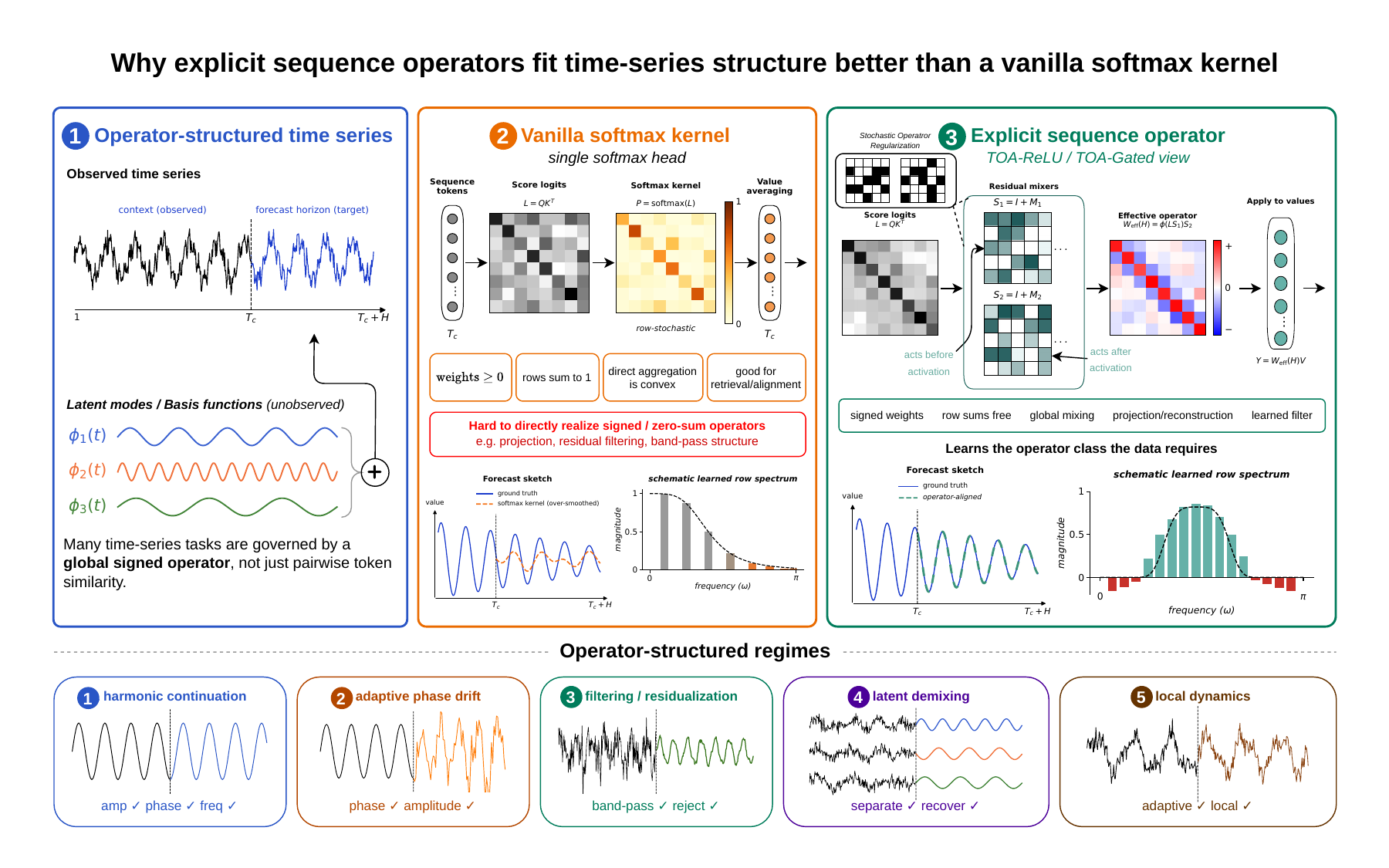}
    \caption{(a) Standard attention relies on simplex-constrained similarity matching, forcing the sequence mixer to act as an implicit low-pass filter that struggles with complex temporal shifts. (b) TOA replaces this implicit routing with an explicit, signed sequence operator. This unlocks negative weights, allowing the model to perform precise band-pass filtering and robust harmonic demixing.}
    \label{fig:architecture}
    \vspace{-15pt}
\end{figure}

\section{Related Works}

Recent work has explored stronger sequence-mixing primitives for time series from several directions. MLP-based models such as TSMixer~\citep{chen2023tsmixerallmlparchitecturetime} and DLinear~\citep{zeng2022transformerseffectivetimeseries} show that explicit mixing over the sequence axis can provide a strong inductive bias, but these operators are typically static and input-independent. To introduce input-dependent adaptivity, attention can be viewed as a form of dynamic weight generation, as emphasized by fast-weight perspectives and recent methods such as HyperMLP~\citep{lu2026hypermlpintegratedperspectivesequence} and related input-conditioned parameterizations, though these often rely on low-rank or otherwise constrained mixing structure. A separate line of work modifies Transformers for time-series forecasting through efficiency improvements, decomposition, or frequency-aware design, as in Informer~\citep{zhou2021informer}, Autoformer~\citep{wu2021autoformer}, FEDformer~\citep{zhou2022fedformerfrequencyenhanceddecomposed}, LogTrans~\citep{li2019enhancing}, and ARM~\citep{lu2026armrefiningmultivariateforecasting}, while still preserving the core similarity-based aggregation mechanism. Our approach is complementary to these lines: TOA targets the sequence-mixing primitive itself by combining input-dependent routing with explicit dense, signed temporal operators, enabling global interactions that are not directly available under standard simplex-constrained attention. A detailed and extended part for related works can be found in Appendix~\ref{extended related works}.

\section{Methodology}
\label{sec:method}

\subsection{Notation and Proof Strategy}
\label{subsec:notation}

We consider sequence-modeling primitives that appear in time-series tasks (forecasting, classification, and anomaly detection) and analyze the sequence-mixing operator shared across their architectures. 

\textbf{Notation.}
Let $\mathbf{x} \in \mathbb{R}^L$ denote the scalar observed sequence and $\mathbf{X} \in \mathbb{R}^{L \times C}$ its multivariate counterpart with $C$ channels. All formal statements are given for the univariate case; multivariate extensions follow by channel composition. After any fixed tokenization, let $\mathbf{H} = [\mathbf{h}_1, \ldots, \mathbf{h}_N]^\top \in \mathbb{R}^{N \times d}$ denote the token matrix, with $N$ positions and embedding dimension $d$. We consider three tokenization schemes used by standard backbones:
\emph{point tokenization} ($N = L$, each $\mathbf{h}_t$ encodes a time step);
\emph{patch tokenization} (PatchTST style: for non-overlapping patches of length $p$, $N = \lfloor L/p \rfloor$, each $\mathbf{h}_i$ encodes a window of $p$ consecutive time steps);
and \emph{inverted tokenization} (iTransformer style: $N = C$, each $\mathbf{h}_c$ encodes the full time series for variate $c$).
We use \emph{operator} for a linear map $\mathbb{R}^N \to \mathbb{R}^N$ that aggregates value tokens, \emph{kernel} for the entry-wise representation $a_{m,n}$ of such an operator, and \emph{predictor} for the end-to-end task map $\mathbf{T}^\star$. Write $\Delta^{N-1} = \{\mathbf{p} \in \mathbb{R}^N_{\geq 0} : \mathbf{1}^\top\mathbf{p} = 1\}$ for the standard $(N-1)$-simplex. We say a property holds \emph{generically} for a parameter $\theta$ ranging over $\mathbb{R}^k$ if it holds outside a Lebesgue-measure-zero subset of $\mathbb{R}^k$. Detailed theoretical derivations regarding the limits of softmax attention across these tokenization schemes can be found in Appendix~\ref{tokenization}.

\textbf{Proof strategy and scope.}
All formal limitation results are stated for a single-layer, single-head setting with scalar values $v_n = h_n^{(1)}$ (the first coordinate of $\mathbf{h}_n$); this convention isolates the primitive-level structural effect cleanly without obscuring it under the value projection $\mathbf{W}_V$. The results do \emph{not} assert universal impossibility for deep multi-head stacks: multi-head softmax can in principle implement signed combinations through $\mathbf{W}_O$, and multi-layer stacks can re-combine signed operators via subsequent MLPs. The empirical motivation for the dense primitive is therefore inductive bias in the finite-data regime, not pure expressivity. 

\subsection{Softmax Attention as Similarity Transport}
\label{subsec:softmax}

For a single self-attention head with projections $\mathbf{W}_Q, \mathbf{W}_K \in \mathbb{R}^{d_h \times d}$ and scalar values $v_n = h_n^{(1)}$, the score matrix
$$ \mathbf{A}(\mathbf{H}) = \mathbf{H}\mathbf{W}_Q^\top \mathbf{W}_K \mathbf{H}^\top \in \mathbb{R}^{N \times N} $$
has entries $\mathbf{A}_{m,n} = \langle \mathbf{W}_Q\mathbf{h}_m,\,\mathbf{W}_K\mathbf{h}_n\rangle$, and is row-normalized via softmax to yield the kernel
$$ a_{m,n}(\mathbf{H}) = \frac{\exp(\mathbf{A}_{m,n})}{\sum_{j=1}^{N} \exp(\mathbf{A}_{m,j})}, \qquad o_m(\mathbf{H}) = \sum_{n=1}^{N} a_{m,n}(\mathbf{H})\, v_n. $$
Each output $o_m(\mathbf{H})$ is therefore a convex combination of value tokens, with mixing weights determined by query-key similarity --- a \emph{similarity transport} interpretation in which information flows along similarity-weighted, mass-preserving paths. The defining structural property is that each row of the kernel lies in the simplex --- $a_{m,n}(\mathbf{H}) \geq 0$ and $\sum_{n=1}^{N} a_{m,n}(\mathbf{H}) = 1$ --- \emph{unconditionally}, regardless of $\mathbf{W}_Q, \mathbf{W}_K$ or input $\mathbf{H}$. This constraint later drives the impossibility result of \S\ref{subsec:impossibility}.
\subsection{Attention as a Dynamic Sequence MLP: Completing the Parameterization}
\label{subsec:mlp_perspective}

A complementary view of self-attention frames it not as similarity routing but as a \emph{dynamically parameterized} sequence-mixing layer operating directly over the sequence dimension. In a static MLP spatial mixer, the output sequence is generated by $\mathbf{o} = \mathbf{W}_{\mathrm{mix}} \mathbf{v}$, where $\mathbf{W}_{\mathrm{mix}} \in \mathbb{R}^{N \times N}$ is a free, learned weight matrix that mixes information across temporal positions. Softmax attention realizes this view through a specific factorization in which the mixing weights are generated dynamically from the input itself:
$$ \mathbf{W}_{\mathrm{mix}}^{\mathrm{soft}}(\mathbf{H}) = \mathrm{softmax}\bigl(\mathbf{H}\mathbf{W}_Q^\top \mathbf{W}_K \mathbf{H}^\top\bigr). $$

This factorized parameterization, however, is structurally incomplete. It restricts every row of $\mathbf{W}_{\mathrm{mix}}^{\mathrm{soft}}(\mathbf{H})$ to the probability simplex for every input $\mathbf{H}$. Moreover, although the post-softmax kernel is generically full-rank for any fixed $\mathbf{H}$, the score matrix $\mathbf{H}\mathbf{W}_Q^\top \mathbf{W}_K \mathbf{H}^\top$ has rank at most $d_h$, so the dependence of $\mathbf{W}_{\mathrm{mix}}^{\mathrm{soft}}$ on $\mathbf{H}$ is mediated entirely by a rank-$d_h$ object. Because $d_h \ll N$, this severely constrains the directions in which the dynamic mixer can vary as $\mathbf{H}$ changes.

To complete the parameterization in this factorized space, we require unconstrained, signed sequence operators that supply the static degrees of freedom the softmax bottleneck removes. This structural approach is directly motivated by recent architectural advances such as HyperMLP~\citep{lu2026hypermlpintegratedperspectivesequence}, which successfully reframed language-model attention as a dynamic MLP. However, to effectively capture time-series dynamics, we must significantly diverge from that framework. 

Specifically, \textbf{Temporal Operator Attention (TOA)} distinguishes itself by: (1) focusing explicitly on the continuous, oscillatory geometry of time-series tasks; (2) utilizing unconstrained, full dense mixing matrices rather than causal DPLR structures; and (3) employing Stochastic Operator Regularization to stabilize the learning of these dense operators, which would otherwise catastrophically overfit in typical time-series forecasting settings. 

We complete the parameterization by introducing learnable static matrices $\mathbf{S}_1, \mathbf{S}_2 \in \mathbb{R}^{N \times N}$ that sandwich the attention activation on the right of the score matrix:
$$ \mathbf{W}_{\mathrm{mix}}^{\mathrm{TOA}}(\mathbf{H}) = \sigma\bigl(\mathbf{H}\mathbf{W}_Q^\top \mathbf{W}_K \mathbf{H}^\top\,\mathbf{S}_1\bigr)\,\mathbf{S}_2, $$
with $\sigma$ a non-saturating, non-normalizing activation (ReLU in TOA-ReLU; a softplus--ReLU product in TOA-Gated). This dual projection formally completes the spatial mixing: $\mathbf{S}_1$ redistributes each query's similarity profile across key positions before the activation, and $\mathbf{S}_2$ applies a signed, full-rank linear map to the gated similarities before value aggregation. The activation $\sigma$ breaks the row-normalization constraint of softmax, and $\mathbf{S}_2$ then supplies the signed degrees of freedom required to leave the nonnegative orthant entirely --- jointly recovering the operators of \S\ref{subsec:regimes}. This yields a hybrid operator combining adaptive context matching with unconstrained, signed mixing.
\subsection{Temporal Operator Attention (TOA)}
\label{subsec:method-toa}

To resolve the incomplete parameterization of softmax attention, we replace the strict row-normalization with learnable dense sequence-mixing operators inserted around the attention activation. Let $\mathbf{M}_1^{(h)}, \mathbf{M}_2^{(h)} \in \mathbb{R}^{N \times N}$ be learnable matrices for head $h$, and define the residual operators
$$ \mathbf{S}_1^{(h)} = \mathbf{I}_N + \mathbf{M}_1^{(h)}, \qquad \mathbf{S}_2^{(h)} = \mathbf{I}_N + \mathbf{M}_2^{(h)}, $$
with entries of $\mathbf{M}_i^{(h)}$ initialized i.i.d.\ from $\mathcal{N}(0, \sigma_M^2)$ at a small scale ($\sigma_M = 10^{-3}$). At initialization $\mathbf{S}_i^{(h)} \approx \mathbf{I}_N$, so signal transport begins from the identity operator and the mixers depart from it as gradients shape $\mathbf{M}_i^{(h)}$ toward task-specific structural filtering.

\textbf{Stochastic operator regularization.}
A primary challenge in applying dense $N \times N$ sequence mixers to time-series tasks is their high propensity to overfit. We address this with a per-forward-pass stochastic regularizer applied directly to the offset $\mathbf{M}_i^{(h)}$, leaving the residual identity $\mathbf{I}_N$ untouched. At each forward pass, we sample a single drop rate $p \sim U[0, 1)$ and an independent Bernoulli mask $\mathbf{B} \in \{0,1\}^{N \times N}$ with entries $B_{jk} \stackrel{\mathrm{iid}}{\sim} \mathrm{Bern}(1 - p)$, and replace the offset with its inverted-dropout image
$$ \tilde{\mathbf{M}}_i^{(h)} = \frac{1}{1 - p}\bigl(\mathbf{M}_i^{(h)} \odot \mathbf{B}\bigr), \qquad \tilde{\mathbf{S}}_i^{(h)} = \mathbf{I}_N + \tilde{\mathbf{M}}_i^{(h)}. $$
Inverted scaling preserves the offset in expectation, $\mathbb{E}[\tilde{\mathbf{M}}_i^{(h)} \mid \mathbf{M}_i^{(h)}, p] = \mathbf{M}_i^{(h)}$, while the per-entry conditional variance $\mathrm{Var}(\tilde{M}_{i,jk}^{(h)} \mid M_{i,jk}^{(h)}, p) = \tfrac{p}{1-p}\,(M_{i,jk}^{(h)})^2$ diverges as $p \to 1^-$. This unboundedness is deliberate: it acts as an implicit regularizer that discourages over-reliance on $\mathbf{M}_i^{(h)}$, since with high probability over $p$ the offset is severely attenuated or scaled to large magnitude. Crucially, because the residual identity is added \emph{outside} the dropout, when $\tilde{\mathbf{M}}_i^{(h)}$ collapses ($p \to 1^-$) the operator reduces to $\tilde{\mathbf{S}}_i^{(h)} \to \mathbf{I}_N$, recovering an identity-mixed fallback that supplies stable signal transport throughout training.

\textbf{Structural Variants and Activation Transfer.}
Having established the pre- and post-activation mixing matrices, we must define the core activation function. We draw direct inspiration from recent findings in language modeling (e.g., HyperMLP~\citep{lu2026hypermlpintegratedperspectivesequence}), which demonstrated that replacing softmax with static ReLU projections or dynamic gating mechanisms significantly enhances expressive capacity. We adapt these two structural choices from their original causal, autoregressive DPLR setting to our dense, bidirectional time-series regime. To rigorously isolate the benefits of unconstrained operators from the choice of activation, we instantiate three variants of TOA (TOA-Softmax, TOA-ReLU, TOA-Gated). Let $d_v$ denote the projected value dimension per head. For a given head $h$, define the pre-activation score matrix as $\mathbf{A}^{(h)} = \mathbf{H}\mathbf{W}_Q^{(h)\top}\mathbf{W}_K^{(h)}\mathbf{H}^\top \in \mathbb{R}^{N \times N}$.

\textbf{Forward pass (TOA-Softmax).}
The most direct modification retains the standard row-normalizing softmax but sandwiches it with our dense sequence operators:
$$ \mathbf{O}^{(h)} = \mathrm{softmax}\bigl(\mathbf{A}^{(h)}\,\mathbf{S}_1^{(h)}\bigr)\,\mathbf{S}_2^{(h)}\;\mathbf{H}\mathbf{W}_V^{(h)\top} \;\in\; \mathbb{R}^{N \times d_v}. $$
This variant serves as a strict structural ablation to determine whether simply adding degrees of freedom via $\mathbf{S}_1$ and $\mathbf{S}_2$ is sufficient, or if the simplex constraint itself must be fundamentally broken.

\textbf{Forward pass (TOA-ReLU).}
Adapting the static regime from prior work, this variant replaces the normalizing softmax with a simple ReLU activation. This variant operates in a \emph{static} regime because the sequence mixing relies on a single, fixed structural pathway per head, applying a constant non-linearity to the score matrix:
$$ \mathbf{O}^{(h)} = \underbrace{\mathrm{ReLU}\!\bigl(\mathbf{A}^{(h)}\,\mathbf{S}_1^{(h)}\bigr)}_{\text{nonlinear score with pre-activation mixing}} \;\mathbf{S}_2^{(h)}\;\underbrace{\mathbf{H}\,\mathbf{W}_V^{(h)\top}}_{\text{values}} \;\in\; \mathbb{R}^{N \times d_v}. $$
By removing the row-normalization constraint, TOA-ReLU allows $\mathbf{S}_2^{(h)}$ to project the nonnegative similarities into a fully unconstrained, signed space, completing the dual parameterization of \S\ref{subsec:mlp_perspective}. 

\textbf{Forward pass (TOA-Gated).}
While TOA-ReLU provides signed mixing, it cannot dynamically alter its structural operator based on the input sequence's macroscopic properties. To adapt dynamic regime selection to our setting, TOA-Gated doubles the QK head count to $2H$ and partitions them into a left group and a right group. For head $h \in \{1,\ldots,H\}$, define
$$ \mathbf{L}^{(h)} = \mathbf{H}\mathbf{W}_{Q,L}^{(h)\top}\mathbf{W}_{K,L}^{(h)}\mathbf{H}^\top, \qquad \mathbf{R}^{(h)} = \mathbf{H}\mathbf{W}_{Q,R}^{(h)\top}\mathbf{W}_{K,R}^{(h)}\mathbf{H}^\top, $$
with independent QK projections per group. The pre-activation mixer $\mathbf{S}_1$ carries independent parameters $\mathbf{S}_1^{(h,L)}, \mathbf{S}_1^{(h,R)}$ across the two groups; the post-activation mixer $\mathbf{S}_2^{(h)}$ uses $H$ heads matching the value path. The per-head pipeline is:
$$ \mathbf{O}^{(h)} = \Bigl(\mathrm{softplus}\!\bigl(\mathbf{R}^{(h)}\,\mathbf{S}_1^{(h,R)}\bigr) \,\odot\, \mathrm{ReLU}\!\bigl(\mathbf{L}^{(h)}\,\mathbf{S}_1^{(h,L)}\bigr)\Bigr)\;\mathbf{S}_2^{(h)}\;\mathbf{H}\,\mathbf{W}_V^{(h)\top}. $$
The left branch supplies a nonnegative similarity profile via ReLU; the right branch supplies a strictly positive sample-dependent gate via softplus. Their elementwise product forms a dynamically gated mixing kernel, upon which $\mathbf{S}_2^{(h)}$ applies a signed, full-rank linear map before final aggregation.

\textbf{Multi-Head Aggregation and Output Projection.}
Following standard attention paradigms, the sequence representations from all $H$ heads are concatenated along the feature dimension and linearly projected to recover the original embedding dimension $d$. The final output of the TOA sequence-mixing layer is:
$$ \mathbf{O} = \mathrm{Concat}\bigl(\mathbf{O}^{(1)}, \dots, \mathbf{O}^{(H)}\bigr)\mathbf{W}_O, $$
where $\mathbf{W}_O \in \mathbb{R}^{H d_v \times d}$ is the learnable output projection matrix. This projection recombines the specialized multi-head features into a unified token space.

\begin{figure}[ht]
    \centering
    % Adjust width as needed (e.g., 1.0\linewidth)
    \includegraphics[width=\linewidth]{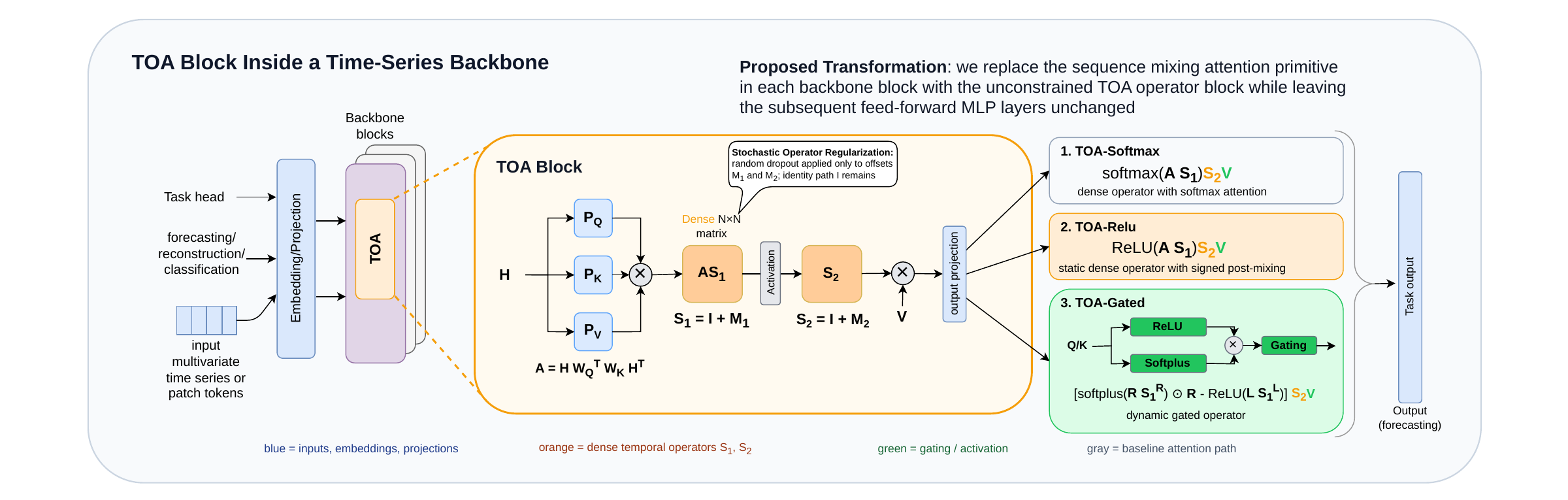}
    \caption{\textbf{The TOA Block}. TOA replaces the standard attention primitive in time-series backbones while leaving the MLP layers unchanged. The internal block architecture (zoomed) illustrates the use of dense, unconstrained sequence operators ($\mathbf{S}_1, \mathbf{S}_2$) to achieve explicit, mixed-sign temporal routing, stabilized via Stochastic Operator Regularization on the learned offsets.}
    \label{fig:model-architecture}
\end{figure}

\subsection{A General Impossibility for Softmax Attention}
\label{subsec:impossibility}

\begin{proposition}[Simplex-constrained mixing limitation]
\label{prop:simplex}
Consider a single-head self-attention layer with scalar values $v_n = h_n^{(1)}$. Let $\mathbf{T}^\star \in \mathbb{R}^{N \times N}$ satisfy $\mathbf{T}^\star_{m,:} \notin \Delta^{N-1}$ for some row $m$. Then for any choice of $\mathbf{W}_Q, \mathbf{W}_K \in \mathbb{R}^{d_h \times d}$, there is no parameter setting such that
\begin{equation}
\label{eq:hypothesis}
    o_m(\mathbf{H}) = \mathbf{T}^\star_{m,:}\,\mathbf{v}(\mathbf{H}) \qquad \text{for all } \mathbf{H} \in \mathbb{R}^{N \times d},
\end{equation}
where $\mathbf{v}(\mathbf{H}) = (h_1^{(1)}, \ldots, h_N^{(1)})^\top$.
\end{proposition}

\textbf{Proof Sketch.}
Suppose toward contradiction that $o_m(\mathbf{H}) = \mathbf{T}^\star_{m,:}\,\mathbf{v}(\mathbf{H})$ for all $\mathbf{H}$. Substituting $\alpha\mathbf{H}$ and dividing by $\alpha > 0$, the values cancel from both sides while the softmax sees logits scaling as $\alpha^2$. As $\alpha \to 0^+$ the kernel collapses to the uniform row $\mathbf{1}^\top/N$, forcing $\tfrac{1}{N}\mathbf{1}^\top\mathbf{v}(\mathbf{H}) = \mathbf{T}^\star_{m,:}\,\mathbf{v}(\mathbf{H})$ for every $\mathbf{H}$. Surjectivity of $\mathbf{H}\mapsto\mathbf{v}(\mathbf{H})$ onto $\mathbb{R}^N$ then forces $\mathbf{T}^\star_{m,:} = \mathbf{1}^\top/N \in \Delta^{N-1}$, contradicting the premise. Full proof in Appendix~\ref{app:proofs}.
\subsection{Time-Series Regimes Requiring Non-Simplex Operators}
\label{subsec:regimes}

We identify five canonical time-series regimes where the optimal sequence operator $\mathbf{T}^\star$ inherently violates the simplex constraint. We show that while softmax attention is mathematically forbidden from reaching these optima, the proposed \textbf{TOA} framework realizes them exactly. Detailed Lemmas with Propositions are provided in Appendix~\ref{app:cases-bd}.

\noindent\textbf{Case A: Seasonal and Harmonic Continuation.}
As proven in Lemma~\ref{lem:app-caseA-failure}, because sampled sinusoids are orthogonal to the DC component~\citep{10.5555/294797}, the optimal operator requires strictly zero-sum rows ($\mathbf{T}^\star \mathbf{1} = \mathbf{0}$), a property unattainable by softmax. Conversely, Proposition~\ref{prop:app-caseA-realization} shows that the unconstrained mixer in \textbf{TOA} natively recovers these kernels.

\noindent\textbf{Case B: Filtering and Residualization.}
Lemma~\ref{lem:app-caseB-failure} proves that high-pass filtering requires negative off-diagonal entries and zero row sums, placing it entirely outside $\Delta^{N-1}$. As shown in Proposition~\ref{prop:app-caseB-realization}, \textbf{TOA} realizes these via the "negative mass" in $\mathbf{M}_2$.

\noindent\textbf{Case C: Multivariate Latent Factor Demixing.}
Lemma~\ref{lem:app-caseC} establishes that for generic mixing matrices $\mathbf{A}$, the MMSE projection $\mathbf{P}_{\mathbf{A}}$ fails simplex membership on three counts: non-unit row sums, non-binary diagonals, and negative entries. Proposition~\ref{prop:app-caseC-realization} proves that \textbf{TOA} can represent these channel-wise projections.

\noindent\textbf{Case D: Phase Drift and Time Warping.}
Lemma~\ref{lem:app-caseD-failure} proves that softmax is unable to adapt to this dynamic family of operators across signal scales. Proposition~\ref{prop:app-caseD-realization} shows that \textbf{TOA-Gated} successfully routes the input to specialized, sign-changing sub-operators.

\noindent\textbf{Case E: Multi-Phase Local Feature Dynamics.}
Lemma~\ref{lem:app-caseE-failure} proves that the optimal rank-2 quadrature projection onto local atoms vanishes under the $\mathbf{1}$-vector. As established in Proposition~\ref{prop:app-caseE-realization}, the localized gating and unconstrained mixer of \textbf{TOA-Gated} allow the architecture to trigger specialized signed filters per transient.
\section{Synthetic Evaluation}
\subsection{Synthetic Tasks: Multi-Regime Harmonic Demixing}
\label{subsubsec:synthetic-demixing}

To validate the theoretical boundaries established in Section~\ref{sec:method}, we evaluate the models' ability to recover oscillatory dynamics from heavily noised, non-stationary signals. 

\begin{wrapfigure}{r}{0.6\textwidth}
    \centering
    \vspace{-15pt}
    
    % --- Column Headers ---
    \begin{minipage}{0.32\linewidth}
        \centering
        \scriptsize\textbf{TOA-Gated}
    \end{minipage}\hfill
    \begin{minipage}{0.32\linewidth}
        \centering
        \scriptsize\textbf{TOA-ReLU}
    \end{minipage}\hfill
    \begin{minipage}{0.32\linewidth}
        \centering
        \scriptsize\textbf{Softmax Attention}
    \end{minipage}

    \vspace{0.5em}

    % --- Row 1 ---
    \begin{minipage}[t]{0.3\linewidth}
        \centering
        \includegraphics[width=\linewidth]{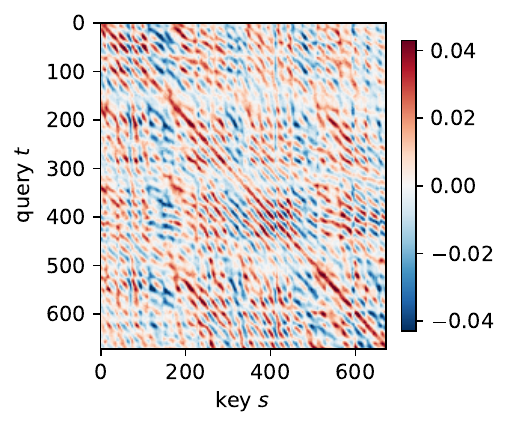}
    \end{minipage}\hfill
    \begin{minipage}[t]{0.3\linewidth}
        \centering
        \includegraphics[width=\linewidth]{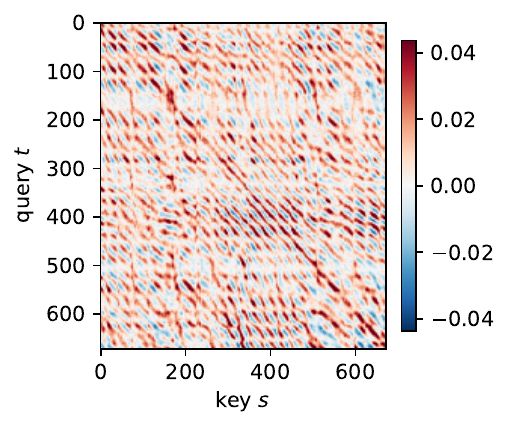}
    \end{minipage}\hfill
    \begin{minipage}[t]{0.3\linewidth}
        \centering
        \includegraphics[width=\linewidth]{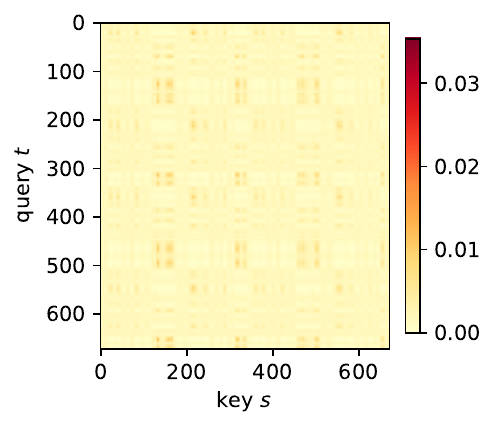}
    \end{minipage}

    \vspace{0.5em}

    % --- Row 2 ---
    \begin{minipage}[t]{0.3\linewidth}
        \centering
        \includegraphics[width=\linewidth]{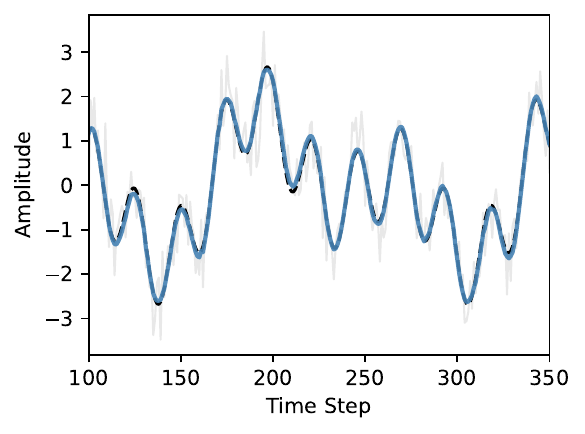}
    \end{minipage}\hfill
    \begin{minipage}[t]{0.3\linewidth}
        \centering
        \includegraphics[width=\linewidth]{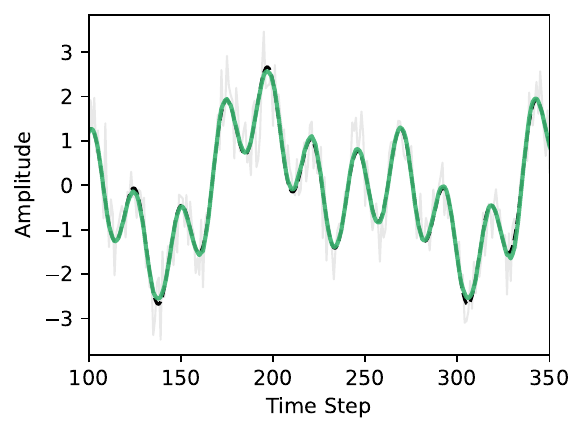}
    \end{minipage}\hfill
    \begin{minipage}[t]{0.3\linewidth}
        \centering
        \includegraphics[width=\linewidth]{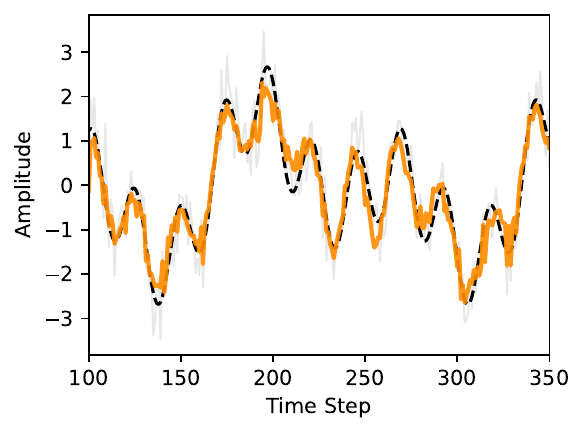}
    \end{minipage}

        % --- Row 3 ---
    \begin{minipage}[t]{0.3\linewidth}
        \centering
        \includegraphics[width=\linewidth]{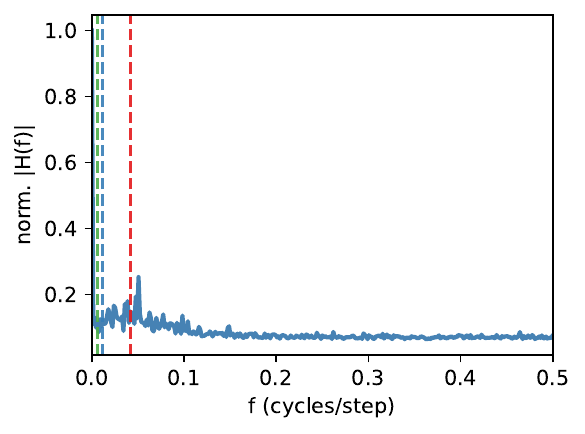}
    \end{minipage}\hfill
    \begin{minipage}[t]{0.3\linewidth}
        \centering
        \includegraphics[width=\linewidth]{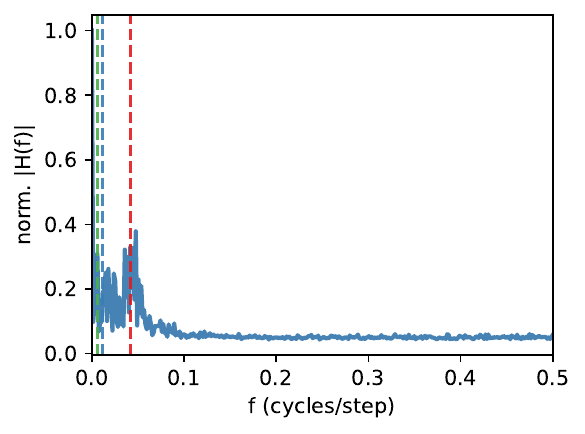}
    \end{minipage}\hfill
    \begin{minipage}[t]{0.3\linewidth}
        \centering
        \includegraphics[width=\linewidth]{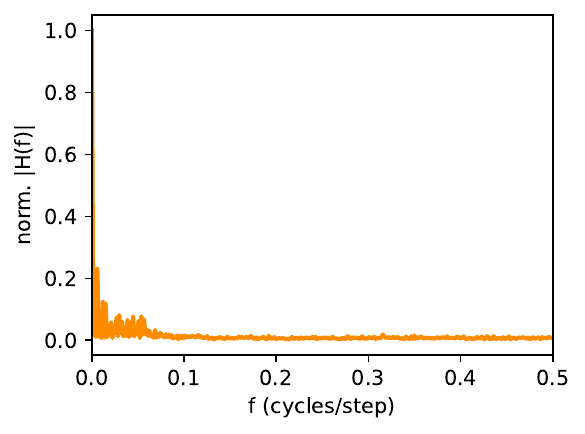}
    \end{minipage}

    \caption{\textbf{Validation of the Simplex Constraint.} Rows show learned sequence operators (top), signal reconstructions (middle), and spectral frequency responses (bottom). TOA-Gated (left) successfully synthesizes mixed-sign weights (blue/red) to subtract massive background interference. Softmax Attention (right) is trapped as a low-pass filter that fails to isolate the target.}
    \label{fig:all_plots_3x3}
    \vspace{-10pt} 
\end{wrapfigure}

\textbf{Setup and Baselines.} We generate a noisy harmonic superposition $x_t = \sum_{k=1}^{3} a_k\cos\!\bigl(\omega_k\tau_z(t) + \phi\bigr) + \varepsilon_t$, with phase drift $\phi \sim \mathrm{Uniform}[0, 2\pi)$ and high-variance noise $\varepsilon_t \sim \mathcal{N}(0, 0.5^2)$. A latent variable $z \in \{0, 1, 2\}$ controls the time-warping regime $\tau_z(t)$: stationary ($\tau_0(t) = t$), periodic vibrato ($\tau_1(t) = t + 20\sin(4\pi t / L)$), or quadratic chirp ($\tau_2(t) = t + 40(t/L)^2$). We use periods $P \in \{24, 84, 168\}$ with length $L = 672$.

To rigorously isolate architectural efficiency, we impose a severe capacity bottleneck of 2 layers and $H=2$ heads. We compare \textbf{TOA-Gated} against \textbf{Softmax Attention} and \textbf{TOA-ReLU}. 

\textbf{Results and Analysis.} Quantitative metrics confirm our theoretical claims. \textbf{TOA-Gated} achieves the lowest MSE (0.0028), outperforming TOA-ReLU (0.0035) and decisively beating Softmax Attention (0.0972).

The learned operator visualizations (Figure~\ref{fig:all_plots_3x3}) confirm our structural claims. Constrained by the probability simplex (Proposition~\ref{prop:simplex}), \textbf{Softmax Attention} collapses into low-pass vertical stripes. While \textbf{TOA-ReLU} learns negative weights via $\mathbf{M}_2$, its static capacity forces destructive interference across warping regimes. Conversely, \textbf{TOA-Gated} dynamically routes inputs to regime-specific sub-operators, yielding cleanly separated filters and demonstrating the necessity of unconstrained, sample-adaptive sequence mixing.
\section{Experimental Evaluation}
\label{sec:experiments}
To empirically validate Temporal Operator Attention (TOA), we conduct experiments across four core tasks: long-term forecasting, short-term forecasting, anomaly detection, and classification. 

\begin{table}[htbp]
    \centering
    \vspace{-10pt} % Tightens space above caption
    \caption{Summary of experimental tasks, backbones, and datasets.}
    \label{tab:exp_summary}
    \vspace{-5pt} % Tightens space below caption
    \resizebox{\linewidth}{!}{%
    \begin{tabular}{@{} l l p{6cm} l @{}}
        \toprule
        \textbf{Task} & \textbf{Backbones} & \textbf{Datasets} & \textbf{Metrics} \\
        \midrule
        LT Forecasting & DUET, PatchTST, iTransformer & Weather, ETT (4 subsets), Electricity, Solar, etc. & MSE, MAE \\
        ST Forecasting & PatchTST, iTransformer & PEMS (4 subsets) & MAE, RMSE \\
        Anomaly Det. & PatchTST, iTransformer, Transformer & PSM, MSL, SMAP, SMD, SWAT & F1-Score \\
        Classification & PatchTST & 28 UEA Multivariate Datasets & Accuracy \\
        \bottomrule
    \end{tabular}%
    }
\end{table}

\subsection{Long- and Short-Term Time Series Forecasting}
\label{subsec:exp_forecasting}

\textbf{Setup and Datasets.} 
We evaluate TOA by integrating it into state-of-the-art backbones. For long-term forecasting, we utilize DUET, PatchTST, and iTransformer across nine datasets. 

\begin{wraptable}{r}{0.55\columnwidth} % Shrunk to 0.55 so text can wrap
    \centering
    \vspace{-15pt} % Adjusted to pull up into the paragraph safely
    \scriptsize 
    \setlength{\tabcolsep}{3pt} 
    \caption{\textbf{PEMS Results} with forecasting horizons $L_{p} \in \{12,24,48,96\}$. Average test MSE. Best in \textbf{bold}, second \underline{underlined}. See Table~\ref{tab:full_results_pems} for full results. }
    \label{mainresult_pems}
    \resizebox{\linewidth}{!}{
    \begin{tabular}{l|cccc|cccc}
        \toprule
        & \multicolumn{4}{c|}{\textbf{PatchTST Backbone}} & \multicolumn{4}{c}{\textbf{iTransformer Backbone}} \\
        \cmidrule(lr){2-5} \cmidrule(l){6-9}
        \textbf{Model} & \textbf{Softmax} & \textbf{Gated} & \textbf{ReLU} & \textbf{TOA-S} & \textbf{Softmax} & \textbf{Gated} & \textbf{ReLU} & \textbf{TOA-S} \\
        \midrule
        PEMS03 & 0.091 & \textbf{0.087} & \underline{0.090} & \textbf{0.087} & \underline{0.150} & \textbf{0.135} & \textbf{0.135} & \underline{0.150} \\
        PEMS04 & \underline{0.094} & \textbf{0.090} & \textbf{0.090} & \textbf{0.090} & 0.115 & \textbf{0.106} & \textbf{0.106} & \underline{0.109} \\
        PEMS07 & \underline{0.067} & 0.069 & \textbf{0.064} & \textbf{0.064} & 0.158 & \underline{0.093} & \textbf{0.092} & 0.138 \\
        PEMS08 & 0.118 & \textbf{0.106} & 0.108 & \underline{0.107} & 0.182 & \underline{0.156} & \textbf{0.154} & 0.183 \\
        \midrule
        AvgRank & 3.75 & \underline{1.75} & 2.00 & \textbf{1.25} & 3.50 & \underline{1.50} & \textbf{1.00} & 3.25 \\
        \#Top1 & 0 & \textbf{3} & \underline{2} & \textbf{3} & 0 & \underline{2} & \textbf{4} & 0 \\
        \bottomrule
    \end{tabular}
    }
    \vspace{-10pt}
\end{wraptable}
To verify that TOA captures localized Dynamics, we extend evaluation to short-term forecasting on PEMS traffic datasets. These datasets on both long- and short-term span a wide array of application areas, each defined by intricate temporal structures and a diverse spectrum of inter-channel dependencies. Detailed experimental setup and detailed dataset descriptions can be found in Appendix~\ref{Exp-details}.

\textbf{Results.} As detailed in Table~\ref{mainresult_pems} and Table~\ref{mainresult_updated}, integrating our gated sequence-mixing operator yields consistent gains. Crucially, by bypassing the normalization bottleneck of standard attention, TOA successfully captures rapid, high-frequency transients that similarity-based mechanisms typically over-smooth. Full experimental results for forecasting tasks can be found in Appendix~\ref{full table MST} at Table~\ref{tab:full_results} and Table~\ref{tab:full_results_pems}.

\begin{table*}[tb]
\caption{Summary of Multivariate TSF Results with forecasting horizons $L_{p} \in \{96,192,336,720\}$. Averaged test set MSE are reported. Best results are in \textbf{bold} and second best are \underline{underlined} (compared within each backbone group). See Table~\ref{tab:full_results} for full results.}
\label{mainresult_updated}
\centering
\begin{scriptsize}
\setlength{\tabcolsep}{1pt}
\resizebox{\textwidth}{!}{
\begin{tabular}{l|cccc|cccc|cccc}
\toprule
\multirow{2}{*}{Model} & \multicolumn{4}{c|}{DUET} & \multicolumn{4}{c|}{PatchTST} & \multicolumn{4}{c}{iTransformer} \\
\cmidrule{2-13}
 & \parbox{1.35cm}{\centering Softmax-Attention} & \parbox{1.35cm}{\centering TOA-Gated} & \parbox{1.35cm}{\centering TOA-ReLU} & \parbox{1.35cm}{\centering TOA-Softmax} & \parbox{1.35cm}{\centering Softmax-Attention} & \parbox{1.35cm}{\centering TOA-Gated} & \parbox{1.35cm}{\centering TOA-ReLU} & \parbox{1.35cm}{\centering TOA-Softmax} & \parbox{1.35cm}{\centering Softmax-Attention} & \parbox{1.35cm}{\centering TOA-Gated} & \parbox{1.35cm}{\centering TOA-ReLU} & \parbox{1.35cm}{\centering TOA-Softmax} \\
\midrule
Weather & 0.250 & \underline{0.243} & \textbf{0.239} & 0.247 & \underline{0.230} & \underline{0.230} & \underline{0.230} & \textbf{0.229} & 0.259 & 0.255 & \underline{0.254} & \textbf{0.252} \\
Solar & 0.231 & 0.231 & \textbf{0.217} & \underline{0.230} & 0.198 & 0.200 & \underline{0.190} & \textbf{0.187} & \textbf{0.234} & 0.242 & 0.240 & \underline{0.238} \\
ECL & 0.177 & \textbf{0.170} & \textbf{0.170} & \underline{0.176} & 0.163 & \underline{0.162} & 0.163 & \textbf{0.160} & 0.180 & 0.176 & \underline{0.175} & \textbf{0.167} \\
ETT & 0.399 & 0.400 & \underline{0.396} & \textbf{0.393} & \underline{0.345} & \underline{0.345} & \underline{0.345} & \textbf{0.342} & 0.385 & \underline{0.380} & \textbf{0.379} & 0.382 \\
Exchange & 0.266 & 0.264 & \textbf{0.258} & \underline{0.262} & 0.673 & \underline{0.567} & \textbf{0.548} & 0.669 & 0.363 & \textbf{0.360} & \underline{0.361} & 0.362 \\
Traffic & 0.466 & \underline{0.464} & 0.470 & \textbf{0.461} & 0.400 & \textbf{0.397} & 0.400 & \underline{0.399} & \underline{0.422} & 0.423 & \underline{0.422} & \textbf{0.420} \\
\midrule
AvgRank & 3.50 & 2.50 & \textbf{1.67} & \underline{2.00} & 2.83 & \underline{2.17} & \underline{2.17} & \textbf{1.50} & 3.17 & 2.83 & \underline{2.00} & \textbf{1.83} \\
\#Top1 & 0 & 1 & \textbf{4} & \underline{2} & 0 & \underline{1} & \underline{1} & \textbf{4} & \underline{1} & \underline{1} & \underline{1} & \textbf{3} \\
\bottomrule
\end{tabular}
}
\end{scriptsize}
\end{table*}

\subsection{Time Series Classification}
\label{subsec:exp_classification}

\begin{wraptable}{r}{0.48\textwidth}
    \centering
    \vspace{-15pt} % Adjusted so it aligns perfectly with the paragraph start
    \footnotesize 
    \setlength{\tabcolsep}{3pt} 
    \caption{\textbf{UEA Results} summary on 28 datasets across 4 different variants. W/T/L vs. Softmax. See Table~\ref{mainresult_classification} for full results.}
    \label{classification_summary}
    \begin{tabular}{@{}lccc@{}}
        \toprule
        \textbf{Model} & \textbf{Acc.} & \textbf{Rank$\downarrow$} & \textbf{W/T/L} \\ \midrule
        Softmax (Base) & 65.62\% & 4.77 & --- \\
        TOA-S (Abl.)   & 64.80\% & 5.11 & 9/2/17 \\ \midrule
        \textbf{TOA-ReLU}  & 69.13\% & \textbf{2.18} & \textbf{26/1/1} \\ 
        \textbf{TOA-Gated} & \textbf{69.16\%} & \underline{2.25} & 24/1/3 \\ \bottomrule
    \end{tabular}
    \vspace{-10pt} 
\end{wraptable}
\textbf{Setup.} To assess representation learning on highly heterogeneous signals, we evaluate TOA on 28 diverse multivariate classification datasets from the UEA archive. These span multiple domains, including high-frequency sensor data and sparse medical diagnostics. Detailed experimental setup can be found in Appendix~\ref{Exp-details}.

\textbf{Results.}
As shown in Table~\ref{classification_summary}, TOA-Gated and TOA-ReLU significantly outperform the Softmax baseline. By utilizing dynamic routing, TOA functions as a sample-adaptive feature extractor capable of identifying discriminant temporal patterns regardless of the data profile. Notably, the TOA-Softmax ablation fails to improve performance, validating that breaking the simplex constraint is necessary for optimal temporal routing. Full experimental results for all 28 datasets for classification analysis can be found in Appendix~\ref{full table MST} at Table~\ref{mainresult_classification}.

\subsection{Time Series Anomaly Detection}
\label{subsec:exp_anomaly}

\textbf{Setup and Datasets.} 
We evaluate the capacity of our model to distinguish between normal dynamics and anomalous deviations using five standard benchmarks: PSM, MSL, SMAP, SMD, and SWAT. Detection performance is measured via the F1-score of reconstructed sequences. Detailed experimental setup can be found in Appendix~\ref{Exp-details}.

\textbf{Results.} 
Table~\ref{mainresult_anomaly} shows that TOA-integrated backbones achieve consistent F1-score improvements across all five anomaly detection benchmarks. By replacing row-normalized similarity transport with explicit, unconstrained sequence operators, our method eliminates the spectral dampening that often obscures transient events in standard attention. This allows for more robust modeling of normal dynamics and produces sharper reconstruction errors for sparse anomalies. TOA-Gated further optimizes this process by dynamically routing out noise, delivering precise and reliable detection across diverse domains. 

\begin{table*}[tb]
\caption{Summary of Anomaly Detection Results. F1-scores (or Accuracy) are reported. Best results are in \textbf{bold} and second best are \underline{underlined} (compared within each backbone group).}
\label{mainresult_anomaly}
\centering
\begin{scriptsize}
\setlength{\tabcolsep}{1pt}
\resizebox{\textwidth}{!}{
\begin{tabular}{l|ccccc|ccccc|ccccc}
\toprule
\multirow{2}{*}{Dataset} & \multicolumn{5}{c|}{iTransformer} & \multicolumn{5}{c|}{PatchTST} & \multicolumn{5}{c}{Transformer} \\
\cmidrule{2-16}
 & \parbox{1.35cm}{\centering Softmax} & \parbox{1.35cm}{\centering TOA-Gated w/o SOR} & \parbox{1.35cm}{\centering TOA-Gated} & \parbox{1.35cm}{\centering TOA-ReLU} & \parbox{1.35cm}{\centering TOA-Softmax} & \parbox{1.35cm}{\centering Softmax} & \parbox{1.35cm}{\centering TOA-Gated w/o SOR} & \parbox{1.35cm}{\centering TOA-Gated} & \parbox{1.35cm}{\centering TOA-ReLU} & \parbox{1.35cm}{\centering TOA-Softmax} & \parbox{1.35cm}{\centering Softmax} & \parbox{1.35cm}{\centering TOA-Gated w/o SOR} & \parbox{1.35cm}{\centering TOA-Gated} & \parbox{1.35cm}{\centering TOA-ReLU} & \parbox{1.35cm}{\centering TOA-Softmax} \\
\midrule
PSM & \underline{0.948} & \textbf{0.949} & \textbf{0.949} & \textbf{0.949} & \underline{0.948} & \underline{0.962} & \underline{0.962} & \underline{0.962} & \textbf{0.968} & \underline{0.962} & 0.906 & \underline{0.907} & \textbf{0.908} & \underline{0.907} & \underline{0.907} \\
MSL & 0.725 & \underline{0.726} & \underline{0.726} & 0.725 & \textbf{0.744} & 0.788 & \textbf{0.794} & \textbf{0.794} & \underline{0.790} & 0.788 & \underline{0.809} & \textbf{0.810} & \textbf{0.810} & \underline{0.809} & \textbf{0.810} \\
SMAP & 0.668 & \textbf{0.670} & \textbf{0.670} & \underline{0.669} & 0.668 & \underline{0.689} & \underline{0.689} & 0.688 & \textbf{0.690} & 0.688 & \textbf{0.737} & 0.733 & \underline{0.734} & 0.733 & \underline{0.734} \\
SMD & \textbf{0.824} & \underline{0.816} & \textbf{0.824} & \textbf{0.824} & \textbf{0.824} & 0.845 & \underline{0.846} & \underline{0.846} & \underline{0.846} & \textbf{0.848} & \underline{0.712} & \underline{0.712} & \underline{0.712} & \underline{0.712} & \textbf{0.713} \\
SWAT & \textbf{0.927} & \textbf{0.927} & \textbf{0.927} & \textbf{0.927} & \underline{0.926} & 0.853 & \underline{0.871} & \textbf{0.897} & 0.860 & 0.853 & 0.803 & \underline{0.815} & 0.814 & \textbf{0.820} & 0.804 \\
\midrule
AvgRank & 2.80 & \underline{2.00} & \textbf{1.20} & \underline{2.00} & 3.00 & 3.40 & \textbf{1.80} & \underline{2.00} & \underline{2.00} & 3.00 & 3.40 & 2.20 & \textbf{1.80} & 2.60 & \underline{2.00} \\
\#Top1 & 2 & \underline{3} & \textbf{4} & \underline{3} & 2 & 0 & \underline{1} & \textbf{2} & \textbf{2} & \underline{1} & \underline{1} & \underline{1} & \textbf{2} & \underline{1} & \textbf{2} \\
\bottomrule
\end{tabular}
}
\end{scriptsize}
\end{table*}
\section{Ablation Studies}

\begin{wraptable}{r}{0.55\columnwidth}
    \centering
    \vspace{-15pt}
    \scriptsize
    \setlength{\tabcolsep}{4pt}
    \caption{\textbf{SOR Ablation.} Test MSE on the PatchTST backbone averaged across $L_{p} \in \{96,192,336,720\}$. Lower MSE is in \textbf{bold} for each architecture variant. See full table in Table~\ref{tab:ablation_results}.}
    \label{tab:ablation_sor}
    \begin{tabular}{l|cc|cc}
        \toprule
        & \multicolumn{2}{c|}{\textbf{TOA-ReLU}} & \multicolumn{2}{c}{\textbf{TOA-Gated}} \\
        \cmidrule(lr){2-3} \cmidrule(l){4-5}
        \textbf{Dataset} & \textbf{w/o SOR} & \textbf{+ SOR} & \textbf{w/o SOR} & \textbf{+ SOR} \\
        \midrule
        Weather & 0.239 & \textbf{0.228} & 0.238 & \textbf{0.230} \\
        ETTm2   & 0.273 & \textbf{0.261} & 0.271 & \textbf{0.259} \\
        ETTm1   & \textbf{0.353} & 0.355 & \textbf{0.352} & 0.355 \\
        ETTh2   & 0.381 & \textbf{0.354} & 0.377 & \textbf{0.360} \\
        ETTh1   & \textbf{0.421} & \textbf{0.421} & 0.423 & \textbf{0.421} \\
        \bottomrule
    \end{tabular}
    \vspace{-10pt}
\end{wraptable}

\textbf{Why Stochastic Operator Regularization?} 
Introducing explicit, unconstrained sequence operators grants the architecture the necessary expressive power to perform true phase-cancellation and band-pass filtering. However, deploying dense $N \times N$ matrices inherently increases the risk of overfitting, particularly on noisy, finite-horizon time series. We conduct an ablation study using the PatchTST backbone to isolate the impact of Stochastic Operator Regularization (SOR). We test both the static and gated variants of Temporal Operator Attention (TOA) to verify that applying high-variance magnitude noise during training prevents the expanded capacity from memorizing spurious high-frequency noise.

\textbf{Ablation Results.} 
Table~\ref{tab:ablation_sor} shows that SOR yields substantial gains, prominently on ETTh2 for TOA-ReLU (0.381 $\rightarrow$ \textbf{0.354}) and ETTm2 for TOA-Gated (0.271 $\rightarrow$ \textbf{0.259}). Despite a marginal regression on ETTm1 (+0.003), the overall trend confirms our hypothesis: unconstrained explicit operators are prone to overfitting on noisy data. By regularizing the operator offset, SOR acts as a vital structural constraint that ensures the learned sequence mixers remain generalizable. Full experimental results for ablation studies can be found in Appendix~\ref{full table MST} at Table~\ref{tab:ablation_results}.
% \section{Discussion}

% The strongest version of this paper avoids broad claims like ``attention is bad for time series.'' The sharper, defensible claim is narrower: (1) attention is a similarity-based transport mechanism; (2) some time-series problems are governed by explicit signed operators over the sequence axis; (3) when this structural mismatch is active --- across at least the five regimes identified in \S\ref{subsec:regimes} --- dense sequence mixing is the correct primitive. This framing also explains why linear and mixer-style models have often been unexpectedly strong in forecasting --- their advantage may not come from simplicity alone, but from more directly parameterizing the operator class many time-series tasks require.

% \textbf{Limitations.} The structural hypothesis may not hold equally across datasets. Highly noisy or weakly periodic data may not benefit much from dense operator learning, and in some cases the more rigid bias of attention or simpler linear models may regularize better. This is not a weakness of the argument --- it is precisely what a structural paper should predict. Proposition~\ref{prop:simplex} is a single-layer single-head statement and does not rule out approximation by deep multi-head stacks; deep softmax stacks empirically do approximate signed operators, at the cost of parameters, depth, and brittleness. Our claim is about the primitive-level inductive bias, not asymptotic expressivity.

\section{Conclusion}
In this work, we identified that standard softmax attention suffers from a simplex-constrained bottleneck, preventing it from modeling the signed and oscillatory dynamics fundamental to time series. To resolve this, we introduced Temporal Operator Attention (TOA), which augments similarity routing with explicit, unconstrained sequence operators. To stabilize these matrices and prevent overfitting, we developed Stochastic Operator Regularization (SOR). Extensive evaluations across forecasting, anomaly detection, and classification confirm that explicitly parameterizing sequence operators provides a critical inductive bias, consistently elevating the performance of state-of-the-art time-series backbones.

\bibliography{TOA}
\bibliographystyle{plainnat}

\clearpage
\appendix

\section{Limitations and Broader Impacts} \label{sec:impacts}

\subsection{Limitations}
While Temporal Operator Attention (TOA) resolves the simplex-constrained bottleneck at the primitive level, it introduces an $O(N^2)$ memory and compute footprint through its dense sequence operators, which limits the maximum sequence length that can be handled directly without token compression or patching (Computation Efficiency Analysis can be found for more details in Appendix~\ref{subsec:efficiency}). In practice, dense signed operators are also more prone to overfitting than standard softmax attention, particularly on smaller datasets, making stable training dependent on effective regularization through our Stochastic Operator Regularization (SOR) scheme. Finally, our formal impossibility results apply specifically to the single-layer, single-head softmax primitive; characterizing the approximation efficiency and sample-complexity tradeoffs of deep, multi-head Transformer stacks relative to TOA remains an important open problem.

\subsection{Impact Statement}
This work studies the mathematical structure of sequence-mixing primitives for time-series modeling and proposes a general-purpose architectural modification for improving temporal representation learning. As a methodological contribution, its societal impact is primarily indirect and depends on downstream application domains. Like other advances in time-series modeling, improved predictive and representational performance could be beneficial in areas such as forecasting, monitoring, and scientific analysis, but may also be deployed in sensitive contexts such as finance, industrial surveillance, or decision support. We therefore view TOA as a general modeling contribution whose real-world impact depends on the transparency, oversight, and domain-specific safeguards of the systems in which it is ultimately used.

\section{Extended Related Works}
\label{extended related works}
\subsection{MLP-based sequence mixing.}
Recent work has demonstrated that explicitly modeling temporal interactions as transformations over the sequence dimension provides a strong inductive bias for time-series data. Architectures such as TSMixer~\citep{chen2023tsmixerallmlparchitecturetime} and DLinear~\citep{zeng2022transformerseffectivetimeseries} replace attention-based aggregation with explicit sequence-space mixing, achieving competitive or superior performance on forecasting benchmarks. However, these approaches are typically static, applying the same transformation regardless of input. In contrast, TOA retains the explicit operator structure of these models while reintroducing input-dependent adaptivity through attention.

Recent work has also explored augmenting static MLP-based forecasting pipelines with learned auxiliary representations of the input sequence. For example, CATS~\citep{lu2026catsenhancingmultivariatetime} constructs auxiliary time series as learned exogenous variables that encode inter-series structure, effectively functioning as a learned preprocessing operator over the temporal dimension. Notably, even with a simple MLP backbone, this approach achieves strong performance, reinforcing the observation that much of the predictive power in multivariate forecasting can be attributed to the design of the sequence-mixing representation rather than deep attention-based architectures.

% \subsection{Dynamic parameterization and fast-weight perspectives.}
% Self-attention can be interpreted as a form of dynamic weight generation, where input-dependent similarity scores define a sequence-mixing operator. This perspective is closely related to fast-weight programming and recent formulations such as HyperMLP~\citep{lu2026hypermlpintegratedperspectivesequence}, which reinterpret attention as a dynamically parameterized sequence mixer. Related mechanistic perspectives, such as Canon-style layers in physics-informed language modeling~\citep{allenzhu2025physicslanguagemodels41}, similarly emphasize the role of structured operator parameterization. These approaches emphasize input-conditioned parameterization, often with low-rank or per-sample structure. In contrast, TOA introduces shared dense operators that explicitly model global temporal structure, complementing dynamic parameterization with a full-rank, signed mixing capability.

\subsection{Dynamic parameterization and fast-weight perspectives.}
Self-attention can be interpreted as a form of dynamic weight generation, where input-dependent similarity scores define a sequence-mixing operator. This perspective is closely related to fast-weight programming and recent formulations such as HyperMLP~\citep{lu2026hypermlpintegratedperspectivesequence}, which reinterpret attention as a dynamically parameterized sequence mixer. Related mechanistic perspectives, such as Canon-style layers in physics-informed language modeling~\citep{allenzhu2025physicslanguagemodels41}, similarly emphasize the role of structured operator parameterization. More broadly, a growing line of work has explored alternative attention parameterizations with greater expressivity than standard softmax mixing. On the linear/recurrent side, Gated Linear Attention and Gated DeltaNet introduce input-dependent gates and richer state updates to increase the flexibility of fast-weight dynamics while preserving efficient recurrence~\citep{yang2024gatedlinearattentiontransformers,yang2025gateddeltanetworksimproving}. In standard softmax attention, Differential Transformer relaxes simplex-style averaging by differencing multiple attention maps to induce signed interactions~\citep{ye2025differentialtransformer}, while Gated Attention for Large Language Models shows that post-attention gating can further improve selectivity and long-context behavior. Recent works such as ZeroS~\citep{lu2026zeroszerosumlinearattention} and Free Energy Mixer~\citep{lu2026freeenergymixer} push this expressivity direction from two complementary angles: ZeroS introduces a zero-sum reparameterization that enables stable signed mixing, whereas Free Energy Mixer replaces expectation-style reads with a value-aware free-energy posterior that enriches channel-wise selection behavior from the same memory bank. These approaches emphasize input-conditioned parameterization, typically in causal, low-rank, recurrent, or value-aware regimes. In contrast, TOA introduces shared dense operators that explicitly model global temporal structure. While previous works rely on low-rank updates or causal gating, TOA provides a full-rank, signed mixing capability that restores the inhibitory degrees of freedom necessary for precise signal cancellation in bidirectional time-series settings.

\subsection{State-space and operator-based sequence models.}

Recent sequence architectures have increasingly replaced explicit attention kernels with structured dynamical operators derived from state-space systems, including S4~\citep{gu2022efficientlymodelinglongsequences}, DSS~\citep{gupta2022diagonalstatespaceseffective}, and Mamba-style selective state-space models~\citep{gu2024mambalineartimesequencemodeling}. These approaches achieve strong long-range modeling performance by parameterizing sequence interactions through continuous-time dynamics, convolutional kernels, or input-dependent recurrence. Conceptually, these approaches share with TOA the broader goal of replacing explicit similarity-based token aggregation with richer parameterized operators over sequences.

However, the underlying inductive biases differ substantially. State-space models typically impose strong structural priors such as HiPPO-based dynamics~\citep{gu2020hipporecurrentmemoryoptimal}, convolutional recurrence, or causal selective scans, designed for efficient long-context propagation. In contrast, TOA targets unrestricted dense temporal mixing directly within the attention primitive itself. Rather than replacing attention with recurrence or implicit convolution, TOA augments the attention pathway with explicit unconstrained signed operators that preserve adaptive context matching while removing the simplex constraint. More broadly, these developments reflect a growing shift from token-wise similarity transport toward explicitly parameterized operators inspired by dynamical systems and signal-processing perspectives on sequence modeling.

From a complementary perspective, recent work has shown that linear attention mechanisms admit an explicit autoregressive interpretation~\citep{lu2025wave}. In particular,~\citet{lu2026lineartransformersvarmodels} demonstrate that single-layer linear attention can be reformulated as a dynamic vector autoregressive (VAR) system, and that deeper linear Transformers can be structurally aligned with autoregressive forecasting objectives through appropriate reparameterization. This further reinforces the connection between sequence-mixing operators and classical dynamical system representations of time series.

However, unlike state-space models that explicitly parameterize latent continuous-time dynamics, these approaches still operate within a constrained linear attention family, preserving a similarity-driven aggregation structure rather than introducing fully unconstrained temporal operators.

\subsection{Transformer variants for time series.}
A large body of work has explored modifications to the Transformer architecture for time-series forecasting, including Informer~\citep{zhou2021informer}, Autoformer~\citep{wu2021autoformer}, and FEDformer~\citep{zhou2022fedformerfrequencyenhanceddecomposed}. These approaches primarily focus on improving efficiency or incorporating frequency-domain structure, while retaining the core similarity-based aggregation mechanism. Similarly, methods such as LogTrans~\citep{li2019enhancing} and ARM~\citep{lu2026armrefiningmultivariateforecasting} introduce convolutional or band-limited structure to attention, biasing it toward local or frequency-aware interactions. While effective in practice, these methods do not fundamentally alter the simplex-constrained nature of the attention kernel. As a result, attention layers remain fundamentally averaging operators, repeatedly mixing representations through row-stochastic interactions. Recent theoretical analyses support this perspective: Dong et al.~\citep{dong2023attentionneedpureattention} show that deep pure-attention architectures can exhibit severe rank collapse with depth, while analogous phenomena in graph neural networks have been linked to over-smoothing under repeated stochastic propagation~\citep{10.5555/3504035.3504468}. Viewed through this lens, standard attention mechanisms may act as implicit spectral smoothers that suppress high-frequency or transient temporal structure.

In contrast, our approach targets the sequence-mixing primitive itself. By introducing explicit, unconstrained operators over the temporal axis, TOA enables global, signed interactions that are not directly accessible within the standard attention formulation. This perspective complements existing efforts by addressing a different axis of the design space: not efficiency or locality, but the structural form of sequence interactions.

\section{Full Proofs for Impossibility and Realization}
\label{app:proofs}

This section provides the complete mathematical proofs for the structural limits of standard softmax attention and the expressive capacity of the dense residual operators, expanding on the proof sketches provided in Section~\ref{sec:method}.

\subsection{Proof of Proposition 1 (The Simplex-Constrained Mixing Limitation)}

\textbf{Proposition 1.} \emph{Consider a single-head self-attention layer with scalar values $v_n = h_n^{(1)}$. Let $\Tstar \in \R^{N \times N}$ satisfy $\Tstar_{m,:} \notin \simplex^{N-1}$ for some row $m$. Then there exist no parameters $\mathbf{W}_Q, \mathbf{W}_K \in \R^{d_h \times d}$ such that $o_m(\Hb) = \Tstar_{m,:}\,\mathbf{v}(\Hb)$ for all $\Hb \in \R^{N \times d}$, where $\mathbf{v}(\Hb) = (h_1^{(1)}, \ldots, h_N^{(1)})^\top$.}

\begin{proof}
Suppose toward contradiction that the hypothesis holds for all $\Hb \in \R^{N \times d}$.

\emph{Step 1: Small-scale limit.}
Fix $\Hb \in \R^{N \times d}$ and consider the scaling family $\{\alpha\Hb : \alpha > 0\}$. The value vector scales linearly, $\mathbf{v}(\alpha\Hb) = \alpha\,\mathbf{v}(\Hb)$, and the output scales as
\begin{equation}
o_m(\alpha\Hb) \;=\; \sum_{n=1}^N a_{m,n}(\alpha\Hb)\,\alpha v_n \;=\; \alpha \sum_{n=1}^N a_{m,n}(\alpha\Hb)\,v_n,
\end{equation}
where $v_n = h_n^{(1)}$. Applying the hypothesis at input $\alpha\Hb$ and dividing both sides by $\alpha > 0$ yields
\begin{equation}
    \label{eq:app-simplex-step1}
    \sum_{n=1}^N a_{m,n}(\alpha\Hb)\, v_n \;=\; \Tstar_{m,:}\,\mathbf{v}(\Hb) \qquad \forall\, \alpha > 0,
\end{equation}
with $\mathbf{v} = \mathbf{v}(\Hb)$ on both sides. Now the $(m,n)$-entry of the score matrix at input $\alpha\Hb$ satisfies
\begin{equation}
    \mathbf{A}_{m,n}(\alpha\Hb) \;=\; \alpha^2\,(\mathbf{W}_Q \mathbf{h}_m)^\top \mathbf{W}_K \mathbf{h}_n \;=\; \alpha^2\,\mathbf{A}_{m,n}(\Hb),
\end{equation}
so every logit scales as $\alpha^2$. Since softmax is continuous and $\mathrm{softmax}(\mathbf{0})_n = 1/N$, we have $a_{m,n}(\alpha\Hb) \to 1/N$ as $\alpha \to 0^+$ for every $n$. Taking the limit of~\eqref{eq:app-simplex-step1} gives
\begin{equation}
    \label{eq:app-simplex-limit}
    \tfrac{1}{N}\mathbf{1}^\top\,\mathbf{v}(\Hb) \;=\; \Tstar_{m,:}\,\mathbf{v}(\Hb).
\end{equation}

\emph{Step 2: Surjectivity and simplex contradiction.}
Equation~\eqref{eq:app-simplex-limit} must hold for every $\Hb \in \R^{N \times d}$. The map $\Hb \mapsto \mathbf{v}(\Hb) = (h_1^{(1)}, \ldots, h_N^{(1)})^\top$ is surjective onto $\R^N$ (its components are $N$ independent free coordinates of $\Hb$), so the linear identity $\tfrac{1}{N}\mathbf{1}^\top\mathbf{v} = \Tstar_{m,:}\,\mathbf{v}$ holds for every $\mathbf{v} \in \R^N$. By uniqueness of representation of linear functionals, $\Tstar_{m,:} = \tfrac{1}{N}\mathbf{1}^\top$. But $\tfrac{1}{N}\mathbf{1}^\top \in \simplex^{N-1}$, contradicting the premise that $\Tstar_{m,:} \notin \simplex^{N-1}$.
\end{proof}

\section{Operator Structure of Time-Series Regimes}
\label{app:cases-bd}

This section provides the complete mathematical formulations and proofs for the five canonical time-series regimes discussed in Section~\ref{sec:method}, demonstrating why they cannot be realized by standard softmax attention and how the TOA framework parameterizes them. Each Realization proposition is stated at the level of the post-mixer family $\{\mathbf{S}_2 = \mathbf{I}_N + \mathbf{M}_2\}$, which, as a simple affine shift, spans all of $\mathbb{R}^{N \times N}$.

\subsection{Case A: Seasonal and Harmonic Continuation}

\textbf{Setup.}
Let $\mathbf{t} = (0, 1, \ldots, L-1)^\top$. For $K$ candidate periods $\{P_k\}$ with angular frequencies $\omega_k = 2\pi/P_k$, define the quadrature feature $\bPhi_k = [\cos(\omega_k\mathbf{t}),\sin(\omega_k\mathbf{t})] \in \mathbb{R}^{L \times 2}$ and the full matrix $\bPhi = [\bPhi_1 \mid \cdots \mid \bPhi_K] \in \mathbb{R}^{L \times 2K}$. The latent signal is $y_t = \sum_{k=1}^{K}\bigl(a_k \cos(\omega_k t) + b_k \sin(\omega_k t)\bigr)$, observed as $x_t = y_t + \varepsilon_t$ where $\varepsilon_t \sim \mathcal{N}(0, \sigma^2)$. Evaluating $\bPhi$ at indices $t = 0,\ldots,L-1$ yields $\bPhi_{\mathrm{obs}}$, and at indices $t = L,\ldots,L+T-1$ yields $\bPhi_{\mathrm{fcast}}$. In the high-SNR limit $\sigma \to 0$, the exact MMSE forecasting operator fits Fourier coefficients by least-squares: $\mathbf{T}^\star = \bPhi_{\mathrm{fcast}}\,\bPhi_{\mathrm{obs}}^+ \in \mathbb{R}^{T \times L}$, where $\bPhi_{\mathrm{obs}}^+ = (\bPhi_{\mathrm{obs}}^\top\bPhi_{\mathrm{obs}})^{-1}\bPhi_{\mathrm{obs}}^\top$. Proposition~\ref{prop:simplex} extends verbatim from self-attention to the cross-attention setting $T \neq L$, since the small-scale scaling argument depends only on the row-stochastic property of softmax.

\begin{lemma}[Failure of Softmax Attention for Harmonic Continuation]
\label{lem:app-caseA-failure}
Suppose (a) $L$ is an integer multiple of $P_k$ for every $k$, (b) all frequencies are non-trivial ($\omega_k \neq 0 \bmod 2\pi$), and (c) $L \geq 2K$ such that $\bPhi_{\mathrm{obs}}$ has full column rank. Then every row of $\mathbf{T}^\star$ has strictly zero sum: $\mathbf{T}^\star \mathbf{1} = \mathbf{0}$. Consequently, no parameter setting for single-head softmax attention (in its cross-attention extension) can realize the harmonic continuation operator.
\end{lemma}

\begin{proof}
Each column of $\bPhi_{\mathrm{obs}}$ represents a sampled sinusoid over an integer number of periods. The discrete sum over the sequence is computed via the geometric series: $\sum_{t=0}^{L-1} e^{i\omega_k t} = \frac{1 - e^{i\omega_k L}}{1 - e^{i\omega_k}} = 0$, as $e^{i\omega_k L} = 1$ and $e^{i\omega_k} \neq 1$. Thus, every column of $\bPhi_{\mathrm{obs}}$ sums to zero: $\mathbf{1}^\top \bPhi_{\mathrm{obs}} = \mathbf{0}^\top$. Right-multiplying by the components of the Moore--Penrose inverse gives $\bPhi_{\mathrm{obs}}^+ \mathbf{1} = (\bPhi_{\mathrm{obs}}^\top\bPhi_{\mathrm{obs}})^{-1}(\bPhi_{\mathrm{obs}}^\top \mathbf{1}) = \mathbf{0}$. It follows that $\mathbf{T}^\star \mathbf{1} = \bPhi_{\mathrm{fcast}}(\bPhi_{\mathrm{obs}}^+ \mathbf{1}) = \mathbf{0}$. Since every row of $\mathbf{T}^\star$ sums to zero, every row fails simplex membership ($\sum_n T^\star_{m,n} = 0 \neq 1$). By Proposition~\ref{prop:simplex}, softmax attention cannot realize this operator.
\end{proof}

\begin{proposition}[Realization by TOA]
\label{prop:app-caseA-realization}
The post-mixer family $\{\mathbf{S}_2 = \mathbf{I}_N + \mathbf{M}_2 : \mathbf{M}_2 \in \mathbb{R}^{N \times N}\}$ contains the harmonic continuation operator $\mathbf{T}^\star$ exactly. In particular, the parameter choice $\mathbf{M}_2^\star = \mathbf{T}^\star - \mathbf{I}_N$ yields $\mathbf{S}_2 = \mathbf{T}^\star$.
\end{proposition}

\begin{proof}
Because the map $\mathbf{M}_2 \mapsto \mathbf{I}_N + \mathbf{M}_2$ is a trivial affine bijection on $\mathbb{R}^{N \times N}$, the post-mixer can represent any arbitrary matrix. Setting the learnable weights to $\mathbf{M}_2^\star = \mathbf{T}^\star - \mathbf{I}_N$ immediately yields $\mathbf{S}_2 = \mathbf{T}^\star$.
\end{proof}

\begin{remark}
\label{rem:app-caseA-pipeline}
Proposition~\ref{prop:app-caseA-realization} establishes that $\mathbf{T}^\star$ lies in the span of the post-mixer alone. In the full TOA-ReLU pipeline $\mathrm{ReLU}(\mathbf{A}\mathbf{S}_1)\mathbf{S}_2\mathbf{H}\mathbf{W}_V^\top$ the left factor is constrained to be entry-wise nonnegative, so realizing $\mathbf{T}^\star$ requires the left factor to be driven to a positive multiple of the identity (or another invertible nonnegative pattern) so that $\mathbf{S}_2$ supplies the signed degrees of freedom.
\end{remark}
\subsection{Case B: Filtering and Residualization}

\textbf{Setup.}
Anomaly detection models typically identify deviations by scoring the residual sequence $\mathbf{r} = (\mathbf{I}_N - \mathbf{T}^\star_{\mathrm{normal}})\mathbf{x}$, where $\mathbf{T}^\star_{\mathrm{normal}}$ represents a smoothing or reconstruction operator for ``normal'' temporal patterns. We analyze the high-pass residualization operator $\mathbf{T}^\star = \mathbf{I}_N - \mathbf{T}^\star_{\mathrm{normal}}$, which generalizes first-order differences, moving-average residuals, and Laplacian-style filtering.

\begin{lemma}[Failure of Softmax Attention for Residualization]
\label{lem:app-caseB-failure}
Let $\mathbf{T}^\star = \mathbf{I}_N - \mathbf{G}$ be a residualization operator where $\mathbf{G} \in \mathbb{R}^{N \times N}$ is any valid smoothing kernel satisfying $\mathbf{G}\mathbf{1} = \mathbf{1}$ and $\mathbf{G} \geq \mathbf{0}$. Then:
\begin{enumerate}[label=(\roman*)]
    \item Every row of $\mathbf{T}^\star$ has zero sum: $\mathbf{T}^\star \mathbf{1} = \mathbf{0}$.
    \item For every row $t$ with $\mathbf{G}_{t,:} \neq \mathbf{e}_t^\top$, the row $\mathbf{T}^\star_{t,:}$ contains at least one strictly negative entry.
\end{enumerate}
Consequently, if $\mathbf{G} \neq \mathbf{I}_N$, at least one row $\mathbf{T}^\star_{t,:} \notin \Delta^{N-1}$, and by Proposition~\ref{prop:simplex} single-head softmax attention cannot realize $\mathbf{T}^\star$.
\end{lemma}

\begin{proof}
(i) Using $\mathbf{G}\mathbf{1} = \mathbf{1}$: $\mathbf{T}^\star \mathbf{1} = (\mathbf{I}_N - \mathbf{G})\mathbf{1} = \mathbf{1} - \mathbf{1} = \mathbf{0}$.

(ii) For any row $t$ with $\mathbf{G}_{t,:} \neq \mathbf{e}_t^\top$, the constraint $\sum_j G_{t,j} = 1$ together with $G_{t,t} \neq 1$ and $G_{t,j} \geq 0$ for all $j$ forces some off-diagonal entry $G_{t,j_0} > 0$ for $j_0 \neq t$. Then $T^\star_{t,j_0} = -G_{t,j_0} < 0$.

If $\mathbf{G} \neq \mathbf{I}_N$, at least one row $t_0$ satisfies $\mathbf{G}_{t_0,:} \neq \mathbf{e}_{t_0}^\top$. By (i)--(ii), the row $\mathbf{T}^\star_{t_0,:}$ has zero sum and a strictly negative entry, so $\mathbf{T}^\star_{t_0,:} \notin \Delta^{N-1}$. Proposition~\ref{prop:simplex} then forbids realization by single-head softmax attention.
\end{proof}

\begin{proposition}[Realization by TOA]
\label{prop:app-caseB-realization}
The post-mixer family $\{\mathbf{S}_2 = \mathbf{I}_N + \mathbf{M}_2 : \mathbf{M}_2 \in \mathbb{R}^{N \times N}\}$ contains every residualization operator $\mathbf{T}^\star = \mathbf{I}_N - \mathbf{G}$ exactly. The parameter choice $\mathbf{M}_2^\star = -\mathbf{G}$ yields $\mathbf{S}_2 = \mathbf{T}^\star$.
\end{proposition}

\begin{proof}
Because the affine map $\mathbf{M}_2 \mapsto \mathbf{I}_N + \mathbf{M}_2$ spans all of $\mathbb{R}^{N \times N}$, we can directly set $\mathbf{M}_2^\star = -\mathbf{G}$, yielding $\mathbf{S}_2 = \mathbf{I}_N - \mathbf{G} = \mathbf{T}^\star$.
\end{proof}

\begin{remark}
\label{rem:app-caseB-pipeline}
The post-mixer alone supplies the unconstrained negative off-diagonal weights and zero row sums required to compute residuals; in the full pipeline, the left ReLU factor must be driven to a nonnegative pattern under which $\mathbf{S}_2$ acts cleanly. 
\end{remark}
\subsection{Case C: Multivariate Latent Factor Demixing}

\textbf{Setup.}
Let $\mathbf{z}_t \in \mathbb{R}^r$ be a latent factor vector ($r < C$) and $\mathbf{A} \in \mathbb{R}^{C \times r}$ a feature mixing matrix such that $\mathbf{x}_t = \mathbf{A}\mathbf{z}_t + \boldsymbol{\varepsilon}_t$. Under iTransformer-style tokenization, the sequence dimension is the channels ($N=C$). The MMSE clean-signal estimator is the orthogonal projection matrix $\mathbf{P}_{\mathbf{A}} = \mathbf{A}(\mathbf{A}^\top\mathbf{A})^{-1}\mathbf{A}^\top$.

\begin{lemma}[Failure of Softmax Attention for Latent Factor Demixing]
\label{lem:app-caseC}
For $\mathbf{A}$ outside a Lebesgue-measure-zero subset of $\mathbb{R}^{C \times r}$ (i.e., generically), each row $c \in \{1,\ldots,C\}$ of $\mathbf{P}_{\mathbf{A}}$ fails simplex membership: (i) the diagonal entry $(\mathbf{P}_{\mathbf{A}})_{c,c} \in (0, 1)$; (ii) some off-diagonal entry in each row is strictly negative on a Lebesgue-positive open subset of parameterizations; (iii) the row sum $(\mathbf{P}_{\mathbf{A}}\mathbf{1})_c \neq 1$. Consequently, by Proposition~\ref{prop:simplex}, softmax attention cannot realize $\mathbf{P}_{\mathbf{A}}$.
\end{lemma}

\begin{proof}
(i) Because $\mathbf{P}_{\mathbf{A}}$ is idempotent and symmetric, $(\mathbf{P}_{\mathbf{A}})_{c,c} = \mathbf{e}_c^\top \mathbf{P}_{\mathbf{A}}\mathbf{e}_c = \|\mathbf{P}_{\mathbf{A}}\mathbf{e}_c\|_2^2 \in [0, 1]$. Equality at 1 requires $\mathbf{e}_c \in \mathrm{col}(\mathbf{A})$, equivalently the augmented matrix $[\mathbf{A}\mid\mathbf{e}_c]$ has rank $r$. Since $C > r$, this defines a subvariety of $\mathbb{R}^{C \times r}$ of codimension at least 1, hence Lebesgue measure zero.

(ii) The off-diagonal entry $(\mathbf{P}_{\mathbf{A}})_{c, c'}$ is a rational function of $\mathbf{A}$ on the open set where $\mathbf{A}^\top\mathbf{A}$ is invertible. For $C=2, r=1, \mathbf{A} = (1, -1)^\top$, $\mathbf{P}_{\mathbf{A}} = \tfrac{1}{2}\bigl(\begin{smallmatrix}1 & -1 \\ -1 & 1\end{smallmatrix}\bigr)$, which has off-diagonal $-1/2 < 0$. By continuity of the rational function, this evaluates to a negative number on an open neighborhood in $\mathbb{R}^{C \times r}$.

(iii) The row sum equals 1 if and only if $\mathbf{P}_{\mathbf{A}}\mathbf{1} = \mathbf{1}$. Because $\mathbf{P}_{\mathbf{A}}$ projects onto $\mathrm{col}(\mathbf{A})$, this requires $\mathbf{1} \in \mathrm{col}(\mathbf{A})$, again a codimension-$(C-r)$ subvariety of measure zero.

Combining (i)--(iii), generically each row of $\mathbf{P}_{\mathbf{A}}$ has a non-unit row sum, a diagonal strictly between 0 and 1, and at least one negative off-diagonal entry, hence $(\mathbf{P}_{\mathbf{A}})_{c,:} \notin \Delta^{C-1}$ for every $c$. Proposition~\ref{prop:simplex} forbids realization by single-head softmax attention.
\end{proof}

\begin{proposition}[Realization by TOA]
\label{prop:app-caseC-realization}
The post-mixer family contains the orthogonal projection $\mathbf{T}^\star = \mathbf{P}_{\mathbf{A}}$ exactly. The parameter choice $\mathbf{M}_2^\star = \mathbf{P}_{\mathbf{A}} - \mathbf{I}_C$ yields $\mathbf{S}_2 = \mathbf{T}^\star$.
\end{proposition}

\begin{proof}
Because the affine map $\mathbf{M}_2 \mapsto \mathbf{I}_C + \mathbf{M}_2$ spans all of $\mathbb{R}^{C \times C}$, we can directly set $\mathbf{M}_2^\star = \mathbf{P}_{\mathbf{A}} - \mathbf{I}_C$, yielding $\mathbf{S}_2 = \mathbf{P}_{\mathbf{A}} = \mathbf{T}^\star$.
\end{proof}

\begin{remark}
\label{rem:app-caseC-pipeline}
Under iTransformer-style tokenization ($N = C$), the post-mixer parameterizes channel-wise projections directly, including the asymmetric and signed entries required by $\mathbf{P}_{\mathbf{A}}$ for generic $\mathbf{A}$.
\end{remark}
\subsection{Case D: Phase Drift and Time Warping}
\label{subsec:caseD}

\textbf{Setup.}
Let $\phi \sim \mathrm{Uniform}[0, 2\pi)$ be a per-sample phase drawn i.i.d., and let $\tau : \mathbb{Z} \to \mathbb{R}$ be a smooth monotone time-warping function (e.g., periodic vibrato or quadratic chirp). The observed signal is $x_t = \sum_{k=1}^{K} a_k\cos(\omega_k\tau(t) + \phi) + \varepsilon_t$. The target $y_{t'}$ is the clean continuation under the same phase and warping regime. In this non-stationary setting, the MMSE predictor is a phase-dependent operator $\mathbf{T}^\star(\phi)$, where $\mathbf{T}^\star(\phi)$ fits the warped quadrature basis $\bPhi(\phi)$ with columns $\cos(\omega_k\tau(t) + \phi)$ and $\sin(\omega_k\tau(t) + \phi)$.

\begin{lemma}[Failure of Softmax Attention for Phase-Adaptive Warping]
\label{lem:app-caseD-failure}
Let $\{\mathbf{T}^\star(\phi)\}_{\phi \in [0, 2\pi)}$ be the family of MMSE predictors under per-sample phase $\phi$ and a fixed smooth monotone warping $\tau$. Assume that for some fixed $\phi_0$, the operator $\mathbf{T}^\star(\phi_0)$ has at least one row outside $\Delta^{N-1}$. Then no parameter setting for single-head softmax attention can realize the family $\{\mathbf{T}^\star(\phi)\}_\phi$.
\end{lemma}

\begin{proof}
By Proposition~\ref{prop:simplex}, no choice of $\mathbf{W}_Q, \mathbf{W}_K$ can realize the single operator $\mathbf{T}^\star(\phi_0)$ as the input-output map of softmax attention. A fortiori, no fixed parameter setting can realize the parameterized family $\{\mathbf{T}^\star(\phi)\}_\phi$, since realizing the family requires in particular realizing the member $\mathbf{T}^\star(\phi_0)$.

The hypothesis that some $\mathbf{T}^\star(\phi_0)$ has a row outside $\Delta^{N-1}$ is generic for the warped quadrature setting: each $\mathbf{T}^\star(\phi)$ is a least-squares projector onto the span of the warped basis $\{\cos(\omega_k\tau(t)+\phi), \sin(\omega_k\tau(t)+\phi)\}_k$, and the row-sum and sign analysis of Lemma~\ref{lem:app-caseA-failure} extends to this basis whenever the warped sample set is approximately balanced over the unit circle, a condition satisfied for non-degenerate $\tau$ and sufficiently long $L$.
\end{proof}

\begin{proposition}[Realization by TOA-Gated]
\label{prop:app-caseD-realization}
For each fixed phase $\phi \in [0, 2\pi)$, the TOA-Gated post-mixer family contains $\mathbf{T}^\star(\phi)$ exactly. Across the family $\{\mathbf{T}^\star(\phi)\}_\phi$, the input-dependent gating product $\mathrm{softplus}(\mathbf{R}\mathbf{S}_1^R) \odot \mathrm{ReLU}(\mathbf{L}\mathbf{S}_1^L)$ supplies the sample-dependent routing required to select among phase-specific operators.
\end{proposition}

\begin{proof}
Because the affine map $\mathbf{M}_2 \mapsto \mathbf{I}_N + \mathbf{M}_2$ spans all of $\mathbb{R}^{N \times N}$, setting $\mathbf{M}_2^\star(\phi) = \mathbf{T}^\star(\phi) - \mathbf{I}_N$ mathematically yields $\mathbf{S}_2 = \mathbf{T}^\star(\phi)$. However, since $\mathbf{M}_2$ is a static parameter, the post-mixer alone realizes only one fixed member of the family. The full TOA-Gated forward pass
$$ \mathbf{O}^{(h)} = \bigl(\mathrm{softplus}(\mathbf{R}\mathbf{S}_1^R) \odot \mathrm{ReLU}(\mathbf{L}\mathbf{S}_1^L)\bigr)\mathbf{S}_2\mathbf{H}\mathbf{W}_V^\top $$
has an input-dependent left factor: the elementwise product depends on $\mathbf{H}$ through $\mathbf{L}, \mathbf{R}$, so the effective output operator $\mathrm{(left)}\cdot\mathbf{S}_2$ varies dynamically with the input. By choosing $\mathbf{S}_2$ to span a basis of phase-specific kernels and training the gate to select among them as a function of $\phi$, the pipeline realizes a sample-dependent operator that matches $\mathbf{T}^\star(\phi)$ on the relevant input distribution.
\end{proof}

\begin{remark}
\label{rem:app-caseD-pipeline}
Unlike Cases A--C, the static post-mixer alone cannot realize the entire family --- the input-dependence of the operator is supplied by the gating product.
\end{remark}
\subsection{Case E: Multi-Phase Local Feature Dynamics}
\label{subsec:caseE}

\textbf{Setup.}
Let local phase atoms be defined as $\psi_{r,\phi,c}(t) = w_r(t - c)\cos(\omega_r(t - c) + \phi)$, where $w_r$ is a strictly positive, compactly supported window of scale $r$, $c$ is the event center, and $\phi \in [0, 2\pi)$ is the local phase. The observed signal is $x_t = \sum_{j} \alpha_j\,\psi_{r_j, \phi_j, c_j}(t) + \varepsilon_t$. This models transient phenomena such as ECG QRS-complexes or mechanical vibration transients. For recovering a single atom at known $(r, c)$ with unknown phase $\phi$, the MMSE denoiser is the rank-2 projection $\mathbf{T}^\star_{\phi,c}$ onto the 2D quadrature subspace at $(r, c)$\footnote{Assuming the quadrature pair constitutes an orthogonal basis of equal norm for the local subspace; for the general case of non-orthogonal or unbalanced atoms, $\mathbf{T}^\star$ involves the inverse of the local Gram matrix.}:
\begin{equation}
    \mathbf{T}^\star_{\phi,c} = \frac{\boldsymbol{\psi}_{\phi,c}\,\boldsymbol{\psi}_{\phi,c}^\top + \boldsymbol{\psi}_{\phi+\pi/2,c}\,\boldsymbol{\psi}_{\phi+\pi/2,c}^\top}{\|\boldsymbol{\psi}_{\phi,c}\|^2 + \|\boldsymbol{\psi}_{\phi+\pi/2,c}\|^2}.
\end{equation}

\begin{lemma}[Failure of Softmax Attention for Local Feature Dynamics]
\label{lem:app-caseE-failure}
Suppose the local atom $\boldsymbol{\psi}_{\phi,c} \in \mathbb{R}^N$ contains entries of both signs and is zero-mean ($\mathbf{1}^\top\boldsymbol{\psi}_{\phi,c} = 0$); the same is assumed for the quadrature companion $\boldsymbol{\psi}_{\phi+\pi/2,c}$. Then for every $(\phi, c)$, single-head softmax attention cannot realize the optimal rank-2 quadrature projector $\mathbf{T}^\star_{\phi,c}$.
\end{lemma}

\begin{proof}
Write $\mathbf{T}^\star_{\phi,c} = (\boldsymbol{\psi}\boldsymbol{\psi}^\top + \boldsymbol{\psi}'\boldsymbol{\psi}'^\top)/Z$ where $\boldsymbol{\psi} = \boldsymbol{\psi}_{\phi,c}$, $\boldsymbol{\psi}' = \boldsymbol{\psi}_{\phi+\pi/2,c}$, and $Z = \|\boldsymbol{\psi}\|^2 + \|\boldsymbol{\psi}'\|^2 > 0$.

\emph{Zero row sum.} For each row $t$, $(\boldsymbol{\psi}\boldsymbol{\psi}^\top)_{t,:}\mathbf{1} = \psi(t)(\boldsymbol{\psi}^\top\mathbf{1}) = 0$ by the zero-mean assumption, and similarly for $\boldsymbol{\psi}'$. Hence $\mathbf{T}^\star_{\phi,c}\mathbf{1} = \mathbf{0}$.

\emph{Sign violation.} Pick any row $t$ with $\psi(t) \neq 0$ or $\psi'(t) \neq 0$. The row of $\mathbf{T}^\star_{\phi,c}$ at $t$ equals $Z^{-1}(\psi(t)\boldsymbol{\psi}^\top + \psi'(t)\boldsymbol{\psi}'^\top)$, a linear combination of two mixed-sign vectors. Generically this row contains at least one strictly negative entry; in particular, the off-diagonal entries of $\boldsymbol{\psi}\boldsymbol{\psi}^\top + \boldsymbol{\psi}'\boldsymbol{\psi}'^\top$ are not uniformly nonnegative for mixed-sign atoms.

Combined, these give $\mathbf{T}^\star_{\phi,c}\mathbf{1} = \mathbf{0} \neq \mathbf{1}$ and at least one strictly negative entry per non-zero row. Therefore some row of $\mathbf{T}^\star_{\phi,c}$ lies outside $\Delta^{N-1}$, and by Proposition~\ref{prop:simplex}, single-head softmax attention cannot realize $\mathbf{T}^\star_{\phi,c}$.
\end{proof}

\begin{proposition}[Realization by TOA-Gated]
\label{prop:app-caseE-realization}
The TOA-Gated forward pass realizes the location- and phase-adaptive operator family $\{\mathbf{T}^\star_{\phi,c}\}_{\phi,c}$ in the same sense as Proposition~\ref{prop:app-caseD-realization}: for each fixed $(\phi, c)$, the post-mixer family contains $\mathbf{T}^\star_{\phi,c}$, and the input-dependent gating product supplies the sample-dependent routing required to select among them.
\end{proposition}

\begin{proof}
The argument parallels Proposition~\ref{prop:app-caseD-realization} with the additional dependence on the temporal center $c$. The gating product $\mathrm{softplus}(\mathbf{R}\mathbf{S}_1^R) \odot \mathrm{ReLU}(\mathbf{L}\mathbf{S}_1^L)$ acts as a localized feature detector over $(t, c)$ pairs: the QK structure of $\mathbf{L}, \mathbf{R}$ encodes which (scale, center) regime the input occupies, and the post-mixer $\mathbf{S}_2$ supplies the signed quadrature projection within each regime.
\end{proof}

\begin{remark}
\label{rem:app-caseE-pipeline}
The combined location-and-phase adaptivity is the principal mechanism by which TOA-Gated outperforms TOA-ReLU on transient-rich signals.
\end{remark}
\section{Theoretical Implications for Advanced Tokenization Schemes}
\label{tokenization}

The foundational limits established in Proposition~\ref{prop:simplex} and their resolution via Temporal Operator Attention (TOA) were established for generic sequence mixers $\mathbb{R}^N \to \mathbb{R}^N$. In modern time-series architectures, the semantic meaning of the sequence length $N$ depends heavily on the chosen tokenization scheme. In this section, we explicitly map our theoretical findings to the two dominant advanced tokenization paradigms: \emph{patch tokenization} (e.g., PatchTST) and \emph{inverted tokenization} (e.g., iTransformer). 

We demonstrate that the structural restrictions of softmax attention severely cripple these state-of-the-art backbones, and prove that integrating TOA mathematically unlocks their intended representational capacities.

\subsection{Patch Tokenization (Macroscopic Phase and Differencing)}
\label{subsec:app-patch}

\textbf{Setup.}
In patch tokenization, the raw scalar sequence $\mathbf{x} \in \mathbb{R}^L$ is strictly partitioned into non-overlapping windows of length $p$. The sequence dimension becomes $N = \lfloor L/p \rfloor$, and each token $\mathbf{h}_n \in \mathbb{R}^d$ encodes the local dynamics of the $n$-th patch. At this macroscopic level, optimal sequence mixing often requires modeling autoregressive transitions, such as discrete derivatives to capture trend changes ($\mathbf{o}_n = \mathbf{h}_n - \mathbf{h}_{n-1}$) or seasonal differencing to isolate cyclic anomalies ($\mathbf{o}_n = \mathbf{h}_n - \mathbf{h}_{n-S}$, where $S$ is the macroscopic seasonal lag). Let $\mathbf{T}^\star \in \mathbb{R}^{N \times N}$ denote a generalized patch-differencing operator.

\begin{lemma}[Failure of Softmax Attention for Inter-Patch Differencing]
\label{lem:app-patch-failure}
Let $\mathbf{T}^\star$ be a patch-differencing operator such that for some target patch $m$ and reference patch $k$, the optimal operation requires subtraction: $\mathbf{T}^\star_{m,m} > 0$ and $\mathbf{T}^\star_{m,k} < 0$. Then $\mathbf{T}^\star_{m,:} \notin \Delta^{N-1}$. Consequently, standard softmax attention over patches cannot compute discrete derivatives or seasonal differences between patch embeddings.
\end{lemma}

\begin{proof}
By definition, the simplex constraint $\Delta^{N-1}$ strictly requires all kernel weights $a_{m,n} \geq 0$. An operation requiring $\mathbf{T}^\star_{m,k} < 0$ forces the target operator outside the non-negative orthant. When standard attention attempts to relate patch $m$ to reference patches, it computes a convex combination $\mathbf{o}_m = \sum_{i=1}^N a_{m,i} \mathbf{h}_i$ subject to $\sum_i a_{m,i} = 1$. Because $a_{m,i} \geq 0$, the model is structurally forced to \emph{smooth} (average) the patches rather than \emph{differentiate} them, mathematically annihilating inter-patch rate-of-change information. By Proposition~\ref{prop:simplex}, softmax attention cannot realize $\mathbf{T}^\star$.
\end{proof}

\begin{proposition}[Realization by TOA Patch-Mixing]
\label{prop:app-patch-realization}
When TOA is applied over patched tokens, the post-activation unconstrained operator $\mathbf{S}_2 \in \mathbb{R}^{N \times N}$ provides the signed degrees of freedom required to compute exact discrete derivatives across patches.
\end{proposition}

\begin{proof}
Unlike standard attention, the Temporal Operator Attention (TOA) module parameterizes the sequence mixing operation via an unconstrained weight matrix $\mathbf{M}_2$, establishing the operator $\mathbf{S}_2 = \mathbf{I}_N + \mathbf{M}_2$. Because this formulation does not enforce a probability simplex $\Delta^{N-1}$, the affine transformation $\mathbf{S}_2$ spans the entirety of $\mathbb{R}^{N \times N}$. Consequently, the residual mixer can be directly optimized such that $(\mathbf{S}_2)_{m,m} = 1$, $(\mathbf{S}_2)_{m,k} = -1$, and zeros elsewhere. This perfectly recovers the inter-patch differencing operator $\mathbf{o}_m = \mathbf{h}_m - \mathbf{h}_k$, moving beyond simplex-constrained mixing and enabling exact macroscopic trend and seasonality analysis.
\end{proof}

\subsection{Inverted Tokenization (Cross-Variate Demixing)}
\label{subsec:app-inverted}

\textbf{Setup.}
In inverted tokenization (iTransformer style), the tokenization axis is flipped. Given a multivariate series $\mathbf{X} \in \mathbb{R}^{L \times C}$, the tokens are formed over the full temporal history of each variate, yielding a sequence length of $N = C$. Each token $\mathbf{h}_n \in \mathbb{R}^d$ represents the entire timeline of variable $n$. In this regime, the sequence mixer acts as a \emph{cross-channel interaction operator}. If the observed variables are entangled linear combinations of true independent physical mechanisms (e.g., Independent Component Analysis), the optimal operator $\mathbf{T}^\star \in \mathbb{R}^{N \times N}$ is a channel-demixing matrix containing both positive and negative correlations.

\begin{lemma}[Failure of Softmax Attention for Cross-Variate Demixing]
\label{lem:app-inverted-failure}
Let $\mathbf{T}^\star$ be a cross-channel demixing operator such that for some target channel $m$ and reference channel $k$, the optimal operation requires subtraction: $\mathbf{T}^\star_{m,m} > 0$ and $\mathbf{T}^\star_{m,k} < 0$. Then $\mathbf{T}^\star_{m,:} \notin \Delta^{N-1}$. Consequently, standard softmax attention over channels cannot compute this cross-variate projection.
\end{lemma}

\begin{proof}
By definition, the simplex constraint $\Delta^{N-1}$ strictly requires all kernel weights to be non-negative. The attention mechanism calculates the output as $\mathbf{o}_m = \sum_{i=1}^N a_{m,i} \mathbf{h}_i$ subject to $\sum_{i=1}^N a_{m,i} = 1$ and $a_{m,i} \geq 0$. An operation requiring $\mathbf{T}^\star_{m,k} < 0$ forces the operator outside the non-negative orthant. As a result, the attention mechanism is structurally forbidden from producing negative coefficients, defaulting instead to a positive correlation bias.
\end{proof}

\begin{proposition}[Realization by TOA Channel-Mixing]
\label{prop:app-inverted-realization}
When TOA is applied under inverted tokenization, the post-activation unconstrained operator $\mathbf{S}_2 \in \mathbb{R}^{N \times N}$ (where $N=C$) provides the signed degrees of freedom required to compute exact cross-variate demixing.
\end{proposition}

\begin{proof}
The TOA sequence operator is inherently parameterized as an affine shift $\mathbf{S}_2 = \mathbf{I}_N + \mathbf{M}_2$, which, free from simplex constraints, trivially spans the entirety of $\mathbb{R}^{N \times N}$. Because $\mathbf{M}_2$ is unconstrained, we can directly set the residual mixer such that $(\mathbf{S}_2)_{m,m} = 1$, $(\mathbf{S}_2)_{m,k} = -1$, and zeros elsewhere. This directly recovers the cross-variate demixing operator $\mathbf{o}_m = \mathbf{h}_m - \mathbf{h}_k$, circumventing the softmax bottleneck and allowing for exact multivariate latent factor demixing.
\end{proof}
\section{Full Multivariate Time Series Analysis Results Table}
\label{full table MST}
\begin{table*}[h!]
\caption{Full Multivariate TSF Results. We report MSE and MAE. Best results are in \textbf{bold} and second best are \underline{underlined} (compared within each backbone group). AvgRank and \#Top1 are calculated based on MSE across all horizons.}
\label{tab:full_results}
\centering
\begin{scriptsize}
\setlength{\tabcolsep}{2pt}
\resizebox{\textwidth}{!}{
\begin{tabular}{ll|cccccccc|cccccccc|cccccccc}
\toprule
\multicolumn{2}{c|}{\multirow{2}{*}{Model}} & \multicolumn{8}{c|}{DUET} & \multicolumn{8}{c|}{PatchTST} & \multicolumn{8}{c}{iTransformer} \\
\cmidrule(lr){3-10} \cmidrule(lr){11-18} \cmidrule(lr){19-26}
\multicolumn{2}{c|}{} & \multicolumn{2}{c}{Softmax-Attention} & \multicolumn{2}{c}{TOA-Gated} & \multicolumn{2}{c}{TOA-ReLU} & \multicolumn{2}{c}{TOA-Softmax} & \multicolumn{2}{c}{Softmax-Attention} & \multicolumn{2}{c}{TOA-Gated} & \multicolumn{2}{c}{TOA-ReLU} & \multicolumn{2}{c}{TOA-Softmax} & \multicolumn{2}{c}{Softmax-Attention} & \multicolumn{2}{c}{TOA-Gated} & \multicolumn{2}{c}{TOA-ReLU} & \multicolumn{2}{c}{TOA-Softmax} \\
Dataset & Len & MSE & MAE & MSE & MAE & MSE & MAE & MSE & MAE & MSE & MAE & MSE & MAE & MSE & MAE & MSE & MAE & MSE & MAE & MSE & MAE & MSE & MAE & MSE & MAE \\
\midrule
Weather & 96 & 0.173 & 0.219 & \underline{0.156} & \underline{0.207} & \textbf{0.154} & \textbf{0.203} & 0.168 & 0.217 & 0.152 & 0.200 & \underline{0.150} & \underline{0.198} & \underline{0.150} & \underline{0.198} & \textbf{0.149} & \textbf{0.196} & 0.175 & 0.215 & \underline{0.170} & \underline{0.211} & \textbf{0.169} & \textbf{0.210} & 0.171 & \underline{0.211} \\
 & 192 & 0.218 & 0.261 & \underline{0.211} & 0.258 & \textbf{0.204} & \textbf{0.252} & 0.216 & \underline{0.257} & 0.195 & 0.241 & 0.195 & \textbf{0.239} & \underline{0.194} & \underline{0.240} & \textbf{0.193} & 0.241 & 0.223 & 0.256 & 0.220 & 0.257 & \underline{0.217} & \textbf{0.254} & \textbf{0.215} & \underline{0.255} \\
 & 336 & 0.269 & 0.300 & \textbf{0.258} & \textbf{0.295} & \underline{0.261} & 0.299 & 0.266 & \underline{0.297} & 0.250 & \underline{0.284} & 0.252 & \underline{0.284} & \underline{0.249} & \textbf{0.283} & \textbf{0.248} & \textbf{0.283} & 0.282 & 0.299 & \underline{0.275} & 0.297 & \underline{0.275} & \underline{0.296} & \textbf{0.272} & \textbf{0.295} \\
 & 720 & \textbf{0.338} & \underline{0.349} & \underline{0.348} & 0.351 & \textbf{0.338} & 0.353 & \textbf{0.338} & \textbf{0.346} & \underline{0.325} & \underline{0.336} & \textbf{0.323} & \textbf{0.333} & \underline{0.325} & \textbf{0.333} & 0.326 & \underline{0.336} & 0.358 & \underline{0.350} & \underline{0.355} & \underline{0.350} & 0.356 & \underline{0.350} & \textbf{0.348} & \textbf{0.346} \\
\midrule
Solar & 96 & 0.199 & 0.250 & 0.212 & \underline{0.236} & \textbf{0.180} & \textbf{0.227} & \underline{0.197} & 0.237 & 0.176 & \textbf{0.228} & 0.195 & 0.282 & \textbf{0.163} & 0.238 & \underline{0.166} & \underline{0.237} & 0.206 & 0.233 & 0.202 & \textbf{0.222} & \underline{0.201} & \textbf{0.222} & \textbf{0.198} & \underline{0.224} \\
 & 192 & 0.234 & 0.276 & 0.246 & \underline{0.263} & \textbf{0.217} & \textbf{0.256} & \underline{0.233} & 0.276 & 0.188 & 0.257 & \textbf{0.185} & \textbf{0.252} & \underline{0.187} & \underline{0.256} & \textbf{0.185} & 0.257 & \textbf{0.231} & 0.257 & 0.245 & 0.257 & \underline{0.237} & \textbf{0.251} & 0.241 & \underline{0.255} \\
 & 336 & 0.247 & 0.283 & \textbf{0.230} & 0.278 & 0.239 & \underline{0.271} & \underline{0.236} & \textbf{0.267} & \underline{0.194} & \underline{0.256} & 0.205 & \underline{0.256} & 0.202 & 0.257 & \textbf{0.188} & \textbf{0.252} & \textbf{0.248} & 0.272 & 0.256 & 0.272 & 0.255 & \textbf{0.269} & \underline{0.252} & \underline{0.270} \\
 & 720 & 0.243 & \textbf{0.274} & \underline{0.237} & 0.278 & \textbf{0.232} & \underline{0.277} & 0.255 & 0.284 & 0.233 & 0.269 & 0.213 & 0.272 & \underline{0.209} & \underline{0.266} & \textbf{0.208} & \textbf{0.262} & \textbf{0.251} & \textbf{0.274} & 0.265 & 0.283 & 0.265 & 0.281 & \underline{0.260} & \underline{0.277} \\
\midrule
ECL & 96 & \underline{0.145} & \underline{0.239} & 0.147 & 0.244 & \textbf{0.141} & \textbf{0.237} & 0.147 & 0.241 & 0.131 & \underline{0.223} & \underline{0.129} & \underline{0.223} & 0.130 & \underline{0.223} & \textbf{0.128} & \textbf{0.221} & \underline{0.147} & \underline{0.239} & \underline{0.147} & \underline{0.239} & \underline{0.147} & \underline{0.239} & \textbf{0.141} & \textbf{0.235} \\
 & 192 & 0.160 & \underline{0.252} & \textbf{0.156} & \textbf{0.251} & \underline{0.157} & 0.255 & 0.161 & 0.253 & 0.150 & 0.241 & \underline{0.149} & \underline{0.240} & \underline{0.149} & 0.241 & \textbf{0.147} & \textbf{0.239} & \underline{0.163} & \underline{0.255} & 0.166 & 0.257 & 0.167 & 0.258 & \textbf{0.161} & \textbf{0.253} \\
 & 336 & 0.178 & \underline{0.272} & \underline{0.172} & 0.274 & \textbf{0.171} & \textbf{0.270} & 0.179 & 0.273 & \textbf{0.165} & \underline{0.259} & \underline{0.166} & \underline{0.259} & 0.169 & 0.264 & \textbf{0.165} & \textbf{0.258} & \underline{0.177} & \underline{0.270} & 0.183 & 0.276 & 0.181 & 0.274 & \textbf{0.173} & \textbf{0.267} \\
 & 720 & 0.224 & 0.312 & \textbf{0.204} & \textbf{0.302} & \underline{0.211} & \underline{0.307} & 0.219 & 0.310 & 0.205 & 0.294 & 0.206 & 0.294 & \underline{0.203} & \textbf{0.291} & \textbf{0.202} & \underline{0.293} & 0.231 & 0.314 & \underline{0.206} & 0.295 & \underline{0.206} & \underline{0.294} & \textbf{0.194} & \textbf{0.287} \\
\midrule
ETTh1 & 96 & 0.399 & 0.415 & 0.400 & \textbf{0.410} & \textbf{0.396} & \underline{0.412} & \underline{0.398} & 0.413 & \textbf{0.370} & \textbf{0.394} & 0.373 & 0.397 & 0.372 & \underline{0.396} & \underline{0.371} & \textbf{0.394} & \underline{0.385} & 0.403 & 0.386 & \underline{0.402} & \textbf{0.384} & \textbf{0.401} & \underline{0.385} & 0.403 \\
 & 192 & 0.452 & \underline{0.442} & 0.447 & 0.444 & \textbf{0.440} & \textbf{0.436} & \underline{0.441} & \underline{0.442} & \textbf{0.407} & \textbf{0.416} & 0.411 & 0.418 & \underline{0.409} & \underline{0.417} & \underline{0.409} & \underline{0.417} & \underline{0.440} & 0.435 & \underline{0.440} & \underline{0.433} & \textbf{0.437} & \textbf{0.430} & 0.442 & 0.436 \\
 & 336 & 0.497 & 0.475 & \textbf{0.473} & \textbf{0.456} & \underline{0.481} & \underline{0.461} & 0.492 & 0.467 & \underline{0.434} & \underline{0.434} & 0.437 & 0.435 & 0.436 & \textbf{0.433} & \textbf{0.433} & \textbf{0.433} & \underline{0.488} & 0.458 & \underline{0.488} & 0.458 & \underline{0.488} & \underline{0.457} & \textbf{0.484} & \textbf{0.456} \\
 & 720 & 0.512 & \underline{0.498} & \textbf{0.496} & \textbf{0.496} & \underline{0.500} & 0.501 & 0.512 & 0.503 & \underline{0.447} & \textbf{0.462} & 0.450 & \underline{0.463} & 0.450 & \underline{0.463} & \textbf{0.446} & \textbf{0.462} & 0.522 & 0.502 & \textbf{0.499} & \textbf{0.486} & \underline{0.501} & \underline{0.487} & 0.508 & 0.494 \\
\midrule
ETTh2 & 96 & 0.317 & 0.361 & \underline{0.306} & \underline{0.355} & 0.307 & 0.356 & \textbf{0.303} & \textbf{0.353} & \textbf{0.274} & \textbf{0.335} & 0.277 & 0.338 & 0.294 & 0.348 & \underline{0.275} & \underline{0.336} & \underline{0.299} & 0.350 & \textbf{0.295} & \textbf{0.346} & \textbf{0.295} & \underline{0.347} & \underline{0.299} & 0.350 \\
 & 192 & 0.416 & 0.419 & 0.397 & 0.408 & \underline{0.394} & \underline{0.405} & \textbf{0.378} & \textbf{0.399} & 0.380 & 0.401 & 0.373 & 0.397 & \textbf{0.339} & \textbf{0.377} & \underline{0.342} & \underline{0.379} & \underline{0.378} & \underline{0.398} & \textbf{0.376} & \textbf{0.396} & \textbf{0.376} & \textbf{0.396} & 0.381 & 0.399 \\
 & 336 & \textbf{0.423} & \textbf{0.432} & \underline{0.438} & 0.440 & 0.439 & \underline{0.437} & 0.448 & 0.446 & 0.390 & 0.413 & 0.389 & 0.413 & \underline{0.383} & \underline{0.411} & \textbf{0.377} & \textbf{0.406} & 0.424 & \underline{0.433} & \underline{0.419} & \textbf{0.430} & \textbf{0.418} & \textbf{0.430} & 0.424 & 0.434 \\
 & 720 & 0.464 & 0.464 & \textbf{0.437} & \textbf{0.447} & \underline{0.440} & \underline{0.448} & \underline{0.440} & 0.452 & \textbf{0.391} & \textbf{0.429} & \underline{0.393} & \underline{0.430} & 0.404 & 0.440 & 0.396 & 0.432 & \underline{0.428} & \underline{0.447} & \textbf{0.425} & \textbf{0.444} & \textbf{0.425} & \textbf{0.444} & \underline{0.428} & \underline{0.447} \\
\midrule
ETTm1 & 96 & \underline{0.354} & \underline{0.382} & 0.371 & 0.394 & 0.358 & 0.387 & \textbf{0.342} & \textbf{0.376} & 0.291 & 0.343 & \underline{0.288} & 0.342 & \textbf{0.283} & \textbf{0.339} & \underline{0.288} & \underline{0.340} & 0.347 & \underline{0.379} & 0.343 & \textbf{0.373} & \textbf{0.333} & \textbf{0.373} & \underline{0.335} & \textbf{0.373} \\
 & 192 & \textbf{0.381} & \textbf{0.398} & 0.403 & \underline{0.411} & 0.404 & 0.412 & \underline{0.382} & \textbf{0.398} & 0.331 & \underline{0.368} & \textbf{0.326} & \textbf{0.366} & 0.332 & 0.372 & \underline{0.329} & \underline{0.368} & 0.380 & 0.394 & 0.377 & \underline{0.392} & \textbf{0.373} & \textbf{0.390} & \underline{0.374} & \underline{0.392} \\
 & 336 & \textbf{0.414} & \textbf{0.418} & 0.432 & 0.429 & 0.441 & 0.434 & \underline{0.422} & \underline{0.425} & 0.370 & 0.394 & \underline{0.366} & \textbf{0.387} & 0.368 & \textbf{0.387} & \textbf{0.365} & \underline{0.391} & 0.418 & 0.418 & \underline{0.408} & \textbf{0.411} & 0.409 & \underline{0.413} & \textbf{0.407} & 0.415 \\
 & 720 & 0.501 & 0.472 & 0.497 & 0.467 & \underline{0.489} & \underline{0.464} & \textbf{0.480} & \textbf{0.461} & 0.415 & 0.422 & \textbf{0.412} & \textbf{0.417} & 0.418 & 0.422 & \underline{0.413} & \underline{0.421} & 0.488 & 0.457 & \underline{0.478} & \underline{0.450} & \textbf{0.472} & \textbf{0.447} & 0.481 & 0.454 \\
\midrule
ETTm2 & 96 & 0.182 & \underline{0.262} & \underline{0.180} & 0.263 & \underline{0.180} & 0.266 & \textbf{0.177} & \textbf{0.258} & \textbf{0.163} & \textbf{0.253} & 0.166 & \underline{0.255} & 0.167 & 0.256 & \underline{0.165} & \textbf{0.253} & 0.184 & 0.269 & \textbf{0.178} & \textbf{0.264} & \underline{0.180} & \underline{0.266} & 0.184 & 0.267 \\
 & 192 & \textbf{0.239} & \underline{0.305} & \underline{0.244} & \textbf{0.304} & 0.261 & 0.325 & 0.247 & 0.309 & \underline{0.222} & \underline{0.295} & \underline{0.222} & \textbf{0.293} & 0.223 & \underline{0.295} & \textbf{0.221} & \textbf{0.293} & 0.253 & 0.313 & \textbf{0.245} & \textbf{0.307} & \underline{0.247} & \underline{0.308} & 0.250 & 0.310 \\
 & 336 & \textbf{0.299} & \textbf{0.342} & 0.350 & 0.379 & 0.363 & 0.388 & \underline{0.338} & \underline{0.369} & \textbf{0.273} & \textbf{0.327} & 0.277 & \underline{0.330} & 0.278 & 0.331 & \underline{0.274} & \textbf{0.327} & \underline{0.313} & \underline{0.351} & \textbf{0.311} & \textbf{0.348} & \textbf{0.311} & \textbf{0.348} & 0.315 & 0.352 \\
 & 720 & 0.527 & 0.472 & 0.533 & 0.480 & \textbf{0.440} & \textbf{0.430} & \underline{0.488} & \underline{0.449} & 0.368 & 0.386 & \textbf{0.360} & \textbf{0.383} & \underline{0.361} & \underline{0.384} & 0.367 & \textbf{0.383} & 0.412 & 0.405 & \underline{0.411} & \underline{0.404} & \textbf{0.410} & \textbf{0.403} & \underline{0.411} & \underline{0.404} \\
\midrule
Exchange & 96 & 0.088 & 0.206 & \textbf{0.080} & \textbf{0.199} & \underline{0.081} & \underline{0.200} & 0.084 & 0.204 & 0.367 & 0.367 & \underline{0.122} & \underline{0.250} & \textbf{0.120} & \textbf{0.248} & 0.149 & 0.270 & \underline{0.087} & 0.207 & \textbf{0.085} & \textbf{0.204} & \textbf{0.085} & \underline{0.205} & \underline{0.087} & 0.207 \\
 & 192 & 0.173 & \underline{0.302} & \textbf{0.168} & \textbf{0.297} & \underline{0.172} & \underline{0.302} & 0.176 & 0.307 & 0.499 & 0.458 & \underline{0.322} & \underline{0.424} & \textbf{0.306} & \textbf{0.411} & 0.460 & 0.458 & \underline{0.178} & \underline{0.301} & \textbf{0.177} & \textbf{0.299} & \underline{0.178} & \textbf{0.299} & \underline{0.178} & \underline{0.301} \\
 & 336 & \textbf{0.277} & \textbf{0.389} & 0.289 & 0.398 & 0.280 & 0.392 & \underline{0.279} & \underline{0.391} & 0.575 & 0.557 & \underline{0.537} & \underline{0.542} & \textbf{0.444} & \textbf{0.509} & 0.774 & 0.625 & 0.333 & \underline{0.419} & \textbf{0.331} & \textbf{0.418} & 0.333 & \underline{0.419} & \underline{0.332} & \textbf{0.418} \\
 & 720 & 0.524 & 0.556 & 0.518 & 0.551 & \textbf{0.499} & \textbf{0.542} & \underline{0.508} & \underline{0.546} & \textbf{1.250} & \textbf{0.821} & \underline{1.289} & \underline{0.833} & 1.322 & 0.863 & 1.294 & 0.849 & 0.853 & 0.696 & \textbf{0.846} & \textbf{0.693} & \underline{0.847} & \underline{0.694} & 0.853 & 0.696 \\
\midrule
Traffic & 96 & \textbf{0.422} & \textbf{0.274} & 0.434 & \textbf{0.274} & 0.448 & \underline{0.278} & \underline{0.431} & \textbf{0.274} & \underline{0.366} & \textbf{0.250} & 0.368 & \underline{0.254} & 0.370 & 0.257 & \textbf{0.365} & \textbf{0.250} & 0.401 & 0.272 & 0.400 & 0.271 & \underline{0.399} & \underline{0.270} & \textbf{0.393} & \textbf{0.267} \\
 & 192 & \underline{0.447} & \underline{0.288} & 0.453 & 0.291 & 0.449 & 0.292 & \textbf{0.442} & \textbf{0.279} & \underline{0.385} & \underline{0.258} & 0.386 & 0.260 & 0.387 & 0.262 & \textbf{0.381} & \textbf{0.256} & \underline{0.415} & \underline{0.278} & 0.417 & 0.279 & 0.416 & 0.280 & \textbf{0.412} & \textbf{0.274} \\
 & 336 & 0.487 & \underline{0.300} & \underline{0.474} & 0.302 & 0.484 & 0.307 & \textbf{0.461} & \textbf{0.286} & \underline{0.398} & \underline{0.266} & 0.400 & 0.267 & 0.408 & 0.277 & \textbf{0.396} & \textbf{0.265} & \underline{0.426} & \textbf{0.287} & 0.428 & 0.291 & 0.427 & \underline{0.289} & \textbf{0.423} & \underline{0.289} \\
 & 720 & 0.509 & \underline{0.308} & \textbf{0.497} & 0.315 & \underline{0.498} & 0.317 & 0.510 & \textbf{0.307} & 0.453 & 0.308 & \textbf{0.434} & \textbf{0.286} & \underline{0.436} & \underline{0.290} & 0.453 & 0.308 & \underline{0.447} & \underline{0.296} & \underline{0.447} & 0.305 & \textbf{0.446} & \textbf{0.295} & 0.450 & \textbf{0.295} \\
\midrule
AvgRank & & \multicolumn{2}{c}{2.89} & \multicolumn{2}{c}{2.42} & \multicolumn{2}{c}{\textbf{2.14}} & \multicolumn{2}{c|}{\underline{2.36}} & \multicolumn{2}{c}{2.67} & \multicolumn{2}{c}{\underline{2.56}} & \multicolumn{2}{c}{2.69} & \multicolumn{2}{c|}{\textbf{1.78}} & \multicolumn{2}{c}{2.94} & \multicolumn{2}{c}{2.22} & \multicolumn{2}{c}{\textbf{1.89}} & \multicolumn{2}{c}{\underline{2.17}} \\
\#Top1 & & \multicolumn{2}{c}{8} & \multicolumn{2}{c}{\underline{10}} & \multicolumn{2}{c}{\textbf{12}} & \multicolumn{2}{c|}{8} & \multicolumn{2}{c}{\underline{8}} & \multicolumn{2}{c}{6} & \multicolumn{2}{c}{6} & \multicolumn{2}{c|}{\textbf{18}} & \multicolumn{2}{c}{3} & \multicolumn{2}{c}{11} & \multicolumn{2}{c}{\textbf{14}} & \multicolumn{2}{c}{\underline{13}} \\
\bottomrule
\end{tabular}
}
\end{scriptsize}
\end{table*}

\begin{table*}[tb]
\caption{Summary of Classification Results. Accuracies are reported. Best results are in \textbf{bold} and second best are \underline{underlined}.}
\label{mainresult_classification}
\centering
\begin{scriptsize}
\setlength{\tabcolsep}{1pt}
\resizebox{\textwidth}{!}{
\begin{tabular}{l|ccccccc}
\toprule
\multirow{2}{*}{Dataset} & \multicolumn{7}{c}{PatchTST} \\
\cmidrule{2-8}
 & \parbox{1.5cm}{\centering Softmax} & \parbox{1.5cm}{\centering HyperGLU (DPLR)} & \parbox{1.5cm}{\centering TOA-Gated} & \parbox{1.5cm}{\centering TOA-Gated w/o SOR} & \parbox{1.5cm}{\centering TOA-ReLU} & \parbox{1.5cm}{\centering TOA-ReLU w/o SOR} & \parbox{1.5cm}{\centering TOA-Softmax} \\
\midrule
ArticularyWordRecognition & \textbf{0.9800} & 0.9756 & \textbf{0.9800} & \underline{0.9767} & \textbf{0.9800} & \underline{0.9767} & \textbf{0.9800} \\
AtrialFibrillation & 0.4000 & 0.4222 & \textbf{0.6000} & 0.3778 & \underline{0.5333} & \underline{0.5333} & 0.4000 \\
BasicMotions & 0.6167 & 0.6083 & \textbf{0.7250} & 0.6167 & \underline{0.7000} & \textbf{0.7250} & 0.6083 \\
CharacterTrajectories & 0.9654 & 0.9584 & \textbf{0.9721} & 0.9631 & \underline{0.9708} & 0.9673 & 0.9652 \\
Cricket & 0.9167 & 0.8981 & 0.9028 & 0.8750 & \underline{0.9306} & \textbf{0.9861} & 0.8981 \\
DuckDuckGeese & 0.2000 & 0.1467 & \underline{0.2600} & 0.1867 & 0.2400 & \textbf{0.2800} & 0.2133 \\
ERing & 0.9457 & 0.8667 & \underline{0.9519} & 0.9370 & \underline{0.9519} & \textbf{0.9556} & 0.8778 \\
Epilepsy & \textbf{0.9589} & 0.9227 & \underline{0.9565} & 0.9275 & 0.9275 & 0.9348 & \underline{0.9565} \\
EthanolConcentration & 0.2421 & 0.2649 & \textbf{0.2814} & 0.2535 & \underline{0.2776} & \textbf{0.2814} & 0.2307 \\
FaceDetection & \underline{0.6642} & 0.6449 & 0.6552 & 0.6467 & \textbf{0.6725} & 0.6524 & 0.6557 \\
FingerMovements & 0.5533 & 0.4867 & \underline{0.6100} & 0.5167 & \textbf{0.6200} & 0.5800 & 0.5333 \\
HandMovementDirection & 0.4820 & 0.4369 & \textbf{0.5405} & 0.4550 & \textbf{0.5405} & \underline{0.5270} & 0.4369 \\
Handwriting & 0.2322 & 0.2259 & 0.2553 & 0.2102 & \textbf{0.2671} & \underline{0.2612} & 0.2267 \\
Heartbeat & 0.6488 & 0.6683 & 0.6829 & 0.6667 & \underline{0.7171} & \textbf{0.7220} & 0.6894 \\
InsectWingbeat & 0.5756 & 0.5857 & \underline{0.6034} & 0.5798 & 0.6024 & \textbf{0.6121} & 0.5718 \\
JapaneseVowels & 0.9441 & 0.9450 & \underline{0.9514} & 0.9441 & \textbf{0.9541} & \underline{0.9514} & 0.9468 \\
LSST & 0.4784 & 0.4674 & 0.4862 & 0.4839 & \textbf{0.5008} & \underline{0.4911} & 0.4757 \\
Libras & 0.7407 & 0.7111 & \underline{0.7722} & 0.7093 & \textbf{0.7778} & 0.7667 & 0.7444 \\
NATOPS & 0.7481 & \underline{0.7593} & \textbf{0.7611} & 0.7370 & 0.7556 & \textbf{0.7611} & \underline{0.7593} \\
PEMS-SF & 0.8170 & 0.7842 & \textbf{0.8728} & 0.7534 & 0.8382 & \underline{0.8555} & 0.7765 \\
PenDigits & 0.9584 & 0.9580 & 0.9694 & 0.9688 & \textbf{0.9711} & \underline{0.9706} & 0.9623 \\
PhonemeSpectra & 0.1179 & 0.0984 & \textbf{0.1265} & 0.0968 & \underline{0.1208} & 0.1133 & 0.1145 \\
RacketSports & 0.7237 & 0.7522 & 0.7697 & 0.7346 & \underline{0.7763} & \textbf{0.7895} & 0.7434 \\
SelfRegulationSCP1 & 0.7873 & 0.8271 & \underline{0.8498} & 0.8089 & 0.8396 & \textbf{0.8532} & 0.8020 \\
SelfRegulationSCP2 & 0.4796 & 0.4944 & \textbf{0.5389} & 0.4907 & \underline{0.5278} & 0.4944 & 0.4648 \\
SpokenArabicDigits & 0.9650 & 0.9767 & \underline{0.9827} & 0.9794 & \textbf{0.9832} & 0.9795 & 0.9674 \\
StandWalkJump & 0.4000 & 0.3556 & \underline{0.4667} & 0.4000 & \textbf{0.5333} & \textbf{0.5333} & 0.3111 \\
UWaveGestureLibrary & 0.8323 & 0.8260 & 0.8406 & 0.8302 & \underline{0.8469} & \textbf{0.8562} & 0.8302 \\
\midrule
AvgRank & 4.768 & 5.696 & \underline{2.250} & 5.679 & \textbf{2.179} & 2.321 & 5.107 \\
\#Top1 & 2 & 0 & 10 & 0 & \underline{11} & \textbf{12} & 1 \\
\bottomrule
\end{tabular}
}
\end{scriptsize}
\end{table*}

\begin{table*}[h!]
\caption{Full PEMS Results. We report MSE and MAE. Best results are in \textbf{bold} and second best are \underline{underlined} (compared within each backbone group). AvgRank and \#Top1 are calculated based on MSE across all horizons.}
\label{tab:full_results_pems}
\centering
\begin{scriptsize}
\setlength{\tabcolsep}{2pt}
\resizebox{\textwidth}{!}{
\begin{tabular}{ll|cccccccc|cccccccc}
\toprule
\multicolumn{2}{c|}{\multirow{2}{*}{Model}} & \multicolumn{8}{c|}{PatchTST} & \multicolumn{8}{c}{iTransformer} \\
\cmidrule(lr){3-10} \cmidrule(lr){11-18}
\multicolumn{2}{c|}{} & \multicolumn{2}{c}{Softmax-Attention} & \multicolumn{2}{c}{TOA-Gated} & \multicolumn{2}{c}{TOA-ReLU} & \multicolumn{2}{c}{TOA-Softmax} & \multicolumn{2}{c}{Softmax-Attention} & \multicolumn{2}{c}{TOA-Gated} & \multicolumn{2}{c}{TOA-ReLU} & \multicolumn{2}{c}{TOA-Softmax} \\
Dataset & Len & MSE & MAE & MSE & MAE & MSE & MAE & MSE & MAE & MSE & MAE & MSE & MAE & MSE & MAE & MSE & MAE \\
\midrule
PEMS03 & 12 & \underline{0.058} & \textbf{0.157} & \textbf{0.057} & \textbf{0.157} & \textbf{0.057} & \textbf{0.157} & \underline{0.058} & \underline{0.159} & \underline{0.066} & \textbf{0.171} & 0.067 & \underline{0.172} & 0.067 & \textbf{0.171} & \textbf{0.064} & \underline{0.172} \\
 & 24 & 0.073 & 0.176 & \underline{0.072} & \textbf{0.174} & \underline{0.072} & \underline{0.175} & \textbf{0.071} & \underline{0.175} & 0.094 & 0.204 & \textbf{0.092} & \textbf{0.200} & \underline{0.093} & \underline{0.202} & 0.096 & 0.207 \\
 & 48 & 0.101 & 0.203 & \textbf{0.094} & \textbf{0.197} & 0.098 & 0.201 & \underline{0.096} & \underline{0.200} & 0.158 & \underline{0.268} & \underline{0.148} & \textbf{0.257} & \textbf{0.147} & \textbf{0.257} & 0.163 & 0.273 \\
 & 96 & 0.131 & 0.232 & \underline{0.125} & \textbf{0.222} & 0.131 & 0.227 & \textbf{0.123} & \underline{0.223} & 0.281 & 0.372 & \textbf{0.234} & \textbf{0.330} & \underline{0.236} & \underline{0.333} & 0.275 & 0.364 \\
\midrule
PEMS04 & 12 & \textbf{0.071} & \textbf{0.169} & \textbf{0.071} & \underline{0.172} & \textbf{0.071} & \underline{0.172} & \underline{0.072} & \underline{0.172} & \underline{0.077} & 0.182 & \textbf{0.076} & \textbf{0.180} & \textbf{0.076} & \underline{0.181} & \underline{0.077} & 0.182 \\
 & 24 & 0.083 & 0.184 & \underline{0.082} & \underline{0.183} & \textbf{0.081} & \textbf{0.181} & \underline{0.082} & 0.185 & 0.095 & 0.204 & \textbf{0.092} & \textbf{0.200} & \textbf{0.092} & \textbf{0.200} & \underline{0.093} & \underline{0.202} \\
 & 48 & 0.100 & 0.202 & 0.096 & \textbf{0.195} & \underline{0.095} & \textbf{0.195} & \textbf{0.094} & \underline{0.196} & 0.129 & 0.242 & \textbf{0.115} & \textbf{0.226} & \underline{0.116} & \textbf{0.226} & 0.119 & \underline{0.232} \\
 & 96 & 0.124 & \underline{0.227} & \underline{0.111} & \textbf{0.211} & \underline{0.111} & \textbf{0.211} & \textbf{0.110} & \textbf{0.211} & 0.160 & 0.271 & \underline{0.141} & \underline{0.251} & \textbf{0.138} & \textbf{0.249} & 0.146 & 0.257 \\
\midrule
PEMS07 & 12 & \textbf{0.048} & \textbf{0.138} & 0.055 & 0.153 & \underline{0.050} & \underline{0.144} & \underline{0.050} & 0.145 & \underline{0.062} & \textbf{0.158} & \underline{0.062} & \underline{0.159} & \textbf{0.061} & \textbf{0.158} & 0.064 & \underline{0.159} \\
 & 24 & \textbf{0.057} & \underline{0.150} & \textbf{0.057} & \textbf{0.147} & \textbf{0.057} & 0.153 & \textbf{0.057} & 0.151 & 0.082 & 0.183 & \underline{0.080} & \textbf{0.181} & \textbf{0.079} & \underline{0.182} & 0.085 & 0.186 \\
 & 48 & 0.071 & 0.168 & 0.072 & 0.170 & \underline{0.069} & \underline{0.166} & \textbf{0.068} & \textbf{0.164} & 0.123 & 0.230 & \underline{0.104} & \underline{0.206} & \textbf{0.102} & \textbf{0.205} & 0.275 & 0.355 \\
 & 96 & 0.091 & 0.193 & 0.092 & 0.195 & \underline{0.082} & \textbf{0.177} & \textbf{0.080} & \underline{0.178} & 0.366 & 0.432 & \underline{0.127} & \textbf{0.230} & \textbf{0.125} & \textbf{0.230} & \underline{0.127} & \underline{0.234} \\
\midrule
PEMS08 & 12 & 0.067 & \underline{0.161} & \textbf{0.064} & \textbf{0.157} & \underline{0.065} & \underline{0.161} & 0.066 & 0.163 & 0.084 & 0.188 & \underline{0.083} & \underline{0.186} & \underline{0.083} & \underline{0.186} & \textbf{0.082} & \textbf{0.184} \\
 & 24 & 0.088 & 0.180 & \underline{0.081} & \textbf{0.172} & \textbf{0.080} & \underline{0.174} & 0.083 & 0.177 & 0.130 & 0.235 & \textbf{0.124} & \textbf{0.228} & \underline{0.125} & \underline{0.229} & \underline{0.125} & \underline{0.229} \\
 & 48 & 0.125 & 0.201 & \textbf{0.117} & \underline{0.190} & \underline{0.122} & \textbf{0.189} & 0.123 & 0.193 & 0.235 & 0.279 & \underline{0.188} & \textbf{0.230} & \textbf{0.185} & \underline{0.231} & 0.267 & 0.308 \\
 & 96 & 0.191 & 0.224 & \underline{0.162} & \textbf{0.205} & 0.164 & \underline{0.207} & \textbf{0.158} & \textbf{0.205} & 0.280 & 0.313 & \underline{0.231} & \underline{0.267} & \textbf{0.223} & \textbf{0.264} & 0.258 & 0.296 \\
\midrule
AvgRank & & \multicolumn{2}{c}{3.19} & \multicolumn{2}{c}{\underline{2.06}} & \multicolumn{2}{c}{\textbf{1.88}} & \multicolumn{2}{c|}{\textbf{1.88}} & \multicolumn{2}{c}{3.38} & \multicolumn{2}{c}{\underline{1.69}} & \multicolumn{2}{c}{\textbf{1.44}} & \multicolumn{2}{c}{3.00} \\
\#Top1 & & \multicolumn{2}{c}{3} & \multicolumn{2}{c}{\underline{6}} & \multicolumn{2}{c}{5} & \multicolumn{2}{c|}{\textbf{8}} & \multicolumn{2}{c}{0} & \multicolumn{2}{c}{\underline{6}} & \multicolumn{2}{c}{\textbf{10}} & \multicolumn{2}{c}{2} \\
\bottomrule
\end{tabular}
}
\end{scriptsize}
\end{table*}

\begin{table*}[h!]
\caption{Full TSF Ablation Results on Stochastic Operator Regularization. We report MSE. Best results are in \textbf{bold} and second best are \underline{underlined} (compared within each group). AvgRank and \#Top1 are calculated based on MSE across all horizons within each group.}
\label{tab:ablation_results}
\centering
\begin{scriptsize}
\setlength{\tabcolsep}{4pt}
\begin{tabular}{ll|cc|cc}
\toprule
\multicolumn{2}{c|}{\multirow{2}{*}{Model}} & \multicolumn{2}{c|}{TOA-Gated} & \multicolumn{2}{c}{TOA-ReLU} \\
\cmidrule{2-6}
\multicolumn{2}{c|}{} & w/o Dropout & w/ Dropout & w/o Dropout & w/ Dropout \\
Dataset & Len & MSE & MSE & MSE & MSE \\
\midrule
Weather & 96 & \underline{0.157} & \textbf{0.154} & \underline{0.154} & \textbf{0.150} \\
 & 192 & \underline{0.202} & \textbf{0.192} & \underline{0.198} & \textbf{0.195} \\
 & 336 & \underline{0.254} & \textbf{0.250} & \underline{0.260} & \textbf{0.247} \\
 & 720 & \underline{0.337} & \textbf{0.324} & \underline{0.343} & \textbf{0.321} \\
\midrule
ETTh1 & 96 & \textbf{0.370} & \underline{0.371} & \underline{0.371} & \textbf{0.370} \\
 & 192 & \underline{0.411} & \textbf{0.408} & \underline{0.412} & \textbf{0.410} \\
 & 336 & \textbf{0.434} & \underline{0.446} & \underline{0.438} & \textbf{0.433} \\
 & 720 & \underline{0.478} & \textbf{0.458} & \textbf{0.462} & \underline{0.470} \\
\midrule
ETTh2 & 96 & \underline{0.299} & \textbf{0.280} & \underline{0.300} & \textbf{0.285} \\
 & 192 & \underline{0.396} & \textbf{0.364} & \underline{0.387} & \textbf{0.354} \\
 & 336 & \underline{0.410} & \textbf{0.407} & \underline{0.419} & \textbf{0.384} \\
 & 720 & \underline{0.404} & \textbf{0.390} & \underline{0.416} & \textbf{0.393} \\
\midrule
ETTm1 & 96 & \underline{0.291} & \textbf{0.291} & \underline{0.295} & \textbf{0.292} \\
 & 192 & \textbf{0.331} & \underline{0.347} & \textbf{0.327} & \underline{0.335} \\
 & 336 & \textbf{0.362} & \underline{0.369} & \textbf{0.370} & \textbf{0.370} \\
 & 720 & \underline{0.423} & \textbf{0.415} & \textbf{0.423} & \underline{0.426} \\
\midrule
ETTm2 & 96 & \underline{0.177} & \textbf{0.166} & \underline{0.179} & \textbf{0.170} \\
 & 192 & \underline{0.245} & \textbf{0.228} & \underline{0.239} & \textbf{0.231} \\
 & 336 & \underline{0.290} & \textbf{0.281} & \underline{0.278} & \textbf{0.277} \\
 & 720 & \underline{0.373} & \textbf{0.362} & \underline{0.397} & \textbf{0.366} \\
\midrule
AvgRank &  & \underline{1.80} & \textbf{1.20} & \underline{1.80} & \textbf{1.15} \\
\#Top1 &  & \underline{4} & \textbf{16} & \underline{4} & \textbf{17} \\
\bottomrule
\end{tabular}
\end{scriptsize}
\end{table*}

\section{Datasets and Experimental Details}
\label{Exp-details}
\subsection{Experimental Datasets}

We evaluate TOA on nine widely used multivariate time series benchmark datasets for long term time series forecasting and 4 datasets for short term time series forecasting. Table \ref{tab:full_results} summarizes the statistics of TOA's performance on these datasets. We provide the full names and key characteristics for each below:

\textbf{Weather Dataset \citep{wu2021autoformer}:} Comprises 21 meteorological indicators (e.g., air temperature, humidity, and wind speed) recorded at 10-minute intervals throughout 2020 at the Max Planck Biogeochemistry Institute.

\textbf{Solar Dataset \citep{lai2018modelinglongshorttermtemporal}:} Tracks solar power production from 137 photovoltaic plants in Alabama, with observations collected every 10 minutes throughout the year 2006.

\textbf{Electricity Dataset \citep{wu2021autoformer}:} Records the hourly electricity consumption of 321 users over a three-year period (2012–2014), capturing both periodic patterns and long-term dependencies.

\textbf{ETT Dataset \citep{zhou2021informer}:} Monitors oil temperature and load metrics from electricity transformers over two years (2016–2018) at both 15-minute (ETTm1, ETTm2) and hourly (ETTh1, ETTh2) resolutions.

\textbf{Exchange Dataset \citep{lai2018modelinglongshorttermtemporal}:} Contains daily exchange rate records for eight countries from 1990 to 2016, providing a longitudinal view of global financial market fluctuations.

\textbf{Traffic Dataset \citep{wu2021autoformer}:} Provides hourly road occupancy rates collected from 862 freeway sensors in the San Francisco Bay Area over a continuous two-year period (2015–2016).

\textbf{PEMS Dataset \citep{li2018diffusionconvolutionalrecurrentneural}:} Captures spatial-temporal traffic dynamics via 5-minute sensor samples in California; this study utilizes the PEMS03, PEMS04, and PEMS08 subsets as standard benchmarks.

For anomaly detection and classification, we follow the standard benchmark suites provided by the Time-Series Library (TSLib). Specifically, anomaly detection is evaluated on the five widely used TSLib benchmarks PSM, MSL, SMAP, SMD, and SWAT, while classification is evaluated on the 28 multivariate datasets from the UEA archive used in the TSLib evaluation pipeline. Unless otherwise noted, we adopt the default dataset splits and preprocessing conventions provided by TSLib for these tasks.

\subsection{Hyper-parameter Settings and Implementation Details}

\textbf{Model Architecture - Forecasting Tasks}. For long- and short-term multivariate forecasting, we integrate Temporal Operator Attention (TOA) into standard time-series backbones, including PatchTST, iTransformer, DUET, and Transformer, by replacing the backbone’s sequence-mixing attention primitive while leaving the surrounding feed-forward blocks unchanged. In all TOA variants, the temporal operators are parameterized as residual dense mixers, $\mathbf{S}_1=\mathbf{I}+\mathbf{M}_1$ and $\mathbf{S}_2=\mathbf{I}+\mathbf{M}_2$, with the learned offsets initialized at small scale ($\sigma_M=10^{-3}$). We evaluate three variants: TOA-Softmax, TOA-Relu, and TOA-Gated. For the forecasting benchmarks, we retain the official or official-aligned backbone hyper-parameters used by each upstream model family. PatchTST follows the standard long-horizon configuration with patch length 16 and stride 8, with dataset-dependent widths and depths taken from the official scripts. iTransformer uses its official aligned encoder-only setup with two encoder layers, eight attention heads, GELU activations, and input normalization enabled. DUET follows its official unified or nearest-official PEMS profile, using mixture-of-experts style sequence modeling with four experts, top-$k=2$, hidden size 256, and channel-independent processing. Prediction horizons are \{96, 192, 336, 720\} for long-horizon forecasting and \{12, 24, 48, 96\} for the PEMS short-horizon benchmarks.

\textbf{Model Architecture - Classification and Anomaly Detection}. For time-series classification and anomaly detection, we use the TSLib benchmark suite with TOA inserted into PatchTST, iTransformer, and Transformer backbones. In classification, the default model width is $d_{\text{model}}=128$ with $d_{\text{ff}}=256$, three encoder layers, eight attention heads, dropout 0.1, and patch length 16 where applicable. The effective input sequence length is set dynamically to the maximum sequence length of each dataset at runtime. For anomaly detection, we use the standard reconstruction-based setup with three encoder layers, eight attention heads, dropout 0.1, and dataset-specific input dimensionalities; typical anomaly configurations use $d_{\text{model}}=128$ and short training schedules of 3 to 10 epochs depending on the dataset. TOA is applied in the same noncausal setting as in forecasting, using identity alignment and backbone-specific maximum sequence lengths for the dense operators.

\textbf{Model Architecture - Synthetic Tasks}. For the synthetic harmonic demixing experiments, we use a deliberately capacity-limited setup to emphasize primitive-level differences in sequence mixing. In particular, the synthetic experiments use shallow backbones with two layers and two attention heads, as described in the main text, so that improvements cannot be attributed simply to scale. The same TOA parameterization is used as in the real-data experiments, with dense residual temporal operators and noncausal sequence mixing.

\textbf{Data Processing Strategy}. For forecasting, we follow the preprocessing conventions of the underlying backbone families and their official scripts. PatchTST uses RevIN-style normalization where specified by the original configuration; iTransformer uses its standard input normalization over the observed context window; and DUET uses channel-independent processing where applicable. All forecasting runs are performed in the multivariate setting. For classification, we use the standard UEA-format pipeline with a validation split ratio of 0.2. For anomaly detection, we use the dataset-specific TSLib preprocessing pipelines and anomaly ratios defined by the benchmark configuration.

\textbf{Training and Hardware}. All training tasks in this paper were conducted using a single H200. For forecasting in the TSF benchmark stack, we use the AdamW optimizer with $\beta=(0.9,0.95)$ and MSE loss. PatchTST and iTransformer use learning rate $10^{-4}$, while DUET uses learning rate $5\times10^{-4}$. Gradient accumulation is used for larger datasets such as Electricity, Solar, Traffic, and the PEMS benchmarks to maintain stable effective batch sizes. For classification, we use the RAdam optimizer with cross-entropy loss, learning rate $10^{-3}$, batch size 16, 100 training epochs, and early stopping patience 10. For anomaly detection, we use the Adam optimizer with MSE reconstruction loss, learning rate $10^{-4}$, batch size typically 128, and early stopping patience 3. In forecasting and anomaly detection, we use random seed 2024, while the classification sweeps are evaluated over three seeds \{2024, 2025, 2026\}.

\section{Computational Efficiency Analysis}
\label{subsec:efficiency}
While replacing row-normalized similarity transport with explicit dense sequence operators inherently introduces an additional parameter footprint, the practical computational overhead remains manageable. A natural theoretical concern with unconstrained $N \times N$ mixers is that they might significantly degrade training throughput, potentially undoing the speed advantages of modern lightweight backbones. 

To contextualize this overhead, we compare the training step times of TOA-augmented models against well-established architectures. As detailed in Table~\ref{tab:efficiency_analysis}, integrating the \textbf{TOA-ReLU} or \textbf{TOA-Gated} mechanisms into PatchTST results in 18.88 ms per training step on the ETTh1 dataset. While this is an expected increase over the base PatchTST model, it remains comparable to standard canonical models such as Informer (19.57 ms) and requires less than half the time of Autoformer (45.55 ms). Similarly, an iTransformer equipped with our dense TOA variants processes a step in 12.92 ms. These results indicate that the representational capacity gained by removing the simplex constraint can be incorporated without exceeding the typical computational budgets established by prior Transformer-based forecasting models.

\begin{table}[htbp]
\centering
\caption{\textbf{Computational Efficiency Comparison.} Average training step time (ms) on the ETTh1 dataset. While the inclusion of explicit dense sequence operators introduces an expected overhead relative to the standard baselines, the overall step time of TOA-augmented backbones remains lower than that of earlier architectures such as Informer and Autoformer.}
\label{tab:efficiency_analysis}
\small
\begin{tabular}{@{}llc@{}}
\toprule
\textbf{Architecture} & \textbf{Sequence Operator} & \textbf{ETTh1 Train Step (ms)} \\
\midrule
Autoformer & Auto-Correlation & 45.55 \\
Informer & ProbSparse Attention & 19.57 \\
\midrule
\multirow{3}{*}{PatchTST} 
& Standard Attention (Baseline) & 10.57 \\
& \textbf{TOA-Softmax} (Ablation) & 13.67 \\
& \textbf{TOA-ReLU / TOA-Gated} & 18.88 \\
\midrule
\multirow{3}{*}{iTransformer} 
& Standard Attention (Baseline) & 6.58 \\
& \textbf{TOA-Softmax} (Ablation) & 9.31 \\
& \textbf{TOA-ReLU / TOA-Gated} & 12.92 \\
\bottomrule
\end{tabular}
\end{table}

\newpage

\end{document}